%% file: Paper.tex
\documentclass[twoside]{article}
\usepackage[accepted]{aistats2022}

\usepackage[round]{natbib}

\usepackage{algorithm}
\usepackage{algorithmic}
\usepackage{amsmath}
\usepackage{amssymb}
\usepackage{amsthm}
\usepackage{bbm}
\usepackage{bm}
\usepackage{caption}
\usepackage{color}
\usepackage{dirtytalk}
\usepackage{dsfont}
\usepackage{enumerate}
\usepackage{graphicx}
\usepackage{listings}
\usepackage{mathtools}
\usepackage{subcaption}
\usepackage{xspace}

\usepackage[usenames,dvipsnames]{xcolor}
\usepackage[bookmarks=false]{hyperref}
\hypersetup{
  pdffitwindow=true,
  pdfstartview={FitH},
  pdfnewwindow=true,
  colorlinks,
  linktocpage=true,
  linkcolor=Green,
  urlcolor=Green,
  citecolor=Green
}
\usepackage[capitalize,noabbrev]{cleveref}

\usepackage[backgroundcolor=white]{todonotes}
\newcommand{\todob}[2][]{\todo[color=red!20,size=\tiny,inline,#1]{B: #2}}

\Crefname{corollary}{Corollary}{Corollaries}
\Crefname{proposition}{Proposition}{Propositions}
\Crefname{theorem}{Theorem}{Theorems}
\Crefname{definition}{Definition}{Definitions}
\Crefname{assumption}{Assumption}{Assumptions}
\Crefname{example}{Example}{Examples}
\Crefname{remark}{Remark}{Remarks}
\Crefname{setting}{Setting}{Settings}
\Crefname{lemma}{Lemma}{Lemmas}

\usepackage{thmtools}
\declaretheorem[name=Theorem,refname={Theorem,Theorems},Refname={Theorem,Theorems}]{theorem}
\declaretheorem[name=Lemma,refname={Lemma,Lemmas},Refname={Lemma,Lemmas},sibling=theorem]{lemma}

\declaretheorem[name=Proposition,refname={Proposition,Propositions},Refname={Proposition,Propositions},sibling=theorem]{proposition}

\newcommand{\E}[1]{\mathbb{E} \left[#1\right]}
\newcommand{\condE}[2]{\mathbb{E} \left[#1 \,\middle|\, #2\right]}

\newcommand{\I}[1]{\mathds{1} \! \left\{#1\right\}}

\mathchardef\mhyphen="2D

\newcommand{\reucb}{\ensuremath{\tt ReUCB}\xspace}
\newcommand{\reucbs}{\ensuremath{\tt ReUCB}$^*$\xspace}
\newcommand{\reucbss}{\ensuremath{\tt ReUCB}$^{**}$\xspace}

\newcommand{\reucbsinf}{\ensuremath{\tt ReUCB}$^{\infty}$\xspace}
\newcommand{\ts}{\ensuremath{\tt TS}\xspace}
\newcommand{\tsm}{\ensuremath{\tt TSm}\xspace}
\newcommand{\tsma}{\ensuremath{\tt TSm1}\xspace}
\newcommand{\tsmb}{\ensuremath{\tt TSm2}\xspace}
\newcommand{\ucb}{\ensuremath{\tt UCB1}\xspace}
\newcommand{\bayesucb}{\ensuremath{\tt Bayes\mhyphen UCB}\xspace}
\newcommand{\klucb}{\ensuremath{\tt KL\mhyphen UCB}\xspace}

\begin{document}

\twocolumn[

\aistatstitle{Random Effect Bandits}

\aistatsauthor{Rong Zhu \And Branislav Kveton}

\aistatsaddress{Institute of Science and Technology for Brain-Inspired  Intelligence \\
Fudan University \And
Amazon$^*$}]

\begin{abstract}
This paper studies regret minimization in a multi-armed bandit. It is well known that side information, such as the prior distribution of arm means in Thompson sampling, can improve the statistical efficiency of the bandit algorithm. While the prior is a blessing when correctly specified, it is a curse when misspecified. To address this issue, we introduce the assumption of a random-effect model to bandits. In this model, the mean arm rewards are drawn independently from an unknown distribution, which we estimate. We derive a random-effect estimator of the arm means, analyze its uncertainty, and design a UCB algorithm \reucb that uses it. We analyze \reucb and derive an upper bound on its $n$-round Bayes regret, which improves upon not using the random-effect structure. Our experiments show that \reucb can outperform Thompson sampling, without knowing the prior distribution of arm means.
\end{abstract}

\input{Introduction}

\input{Setting}

\input{ModelEstimation}

\input{Algorithm}

\input{Analysis}

\input{Experiments}

\input{Conclusions}

\subsubsection*{Acknowledgements}

This work was partially supported by Science and Technology Innovation 2030 - Brain Science and Brain-Inspired Intelligence Project (No. 2021ZD0200204), by National Natural Science Foundation of China (No.11871459), and by Shanghai Municipal Science and Technology Major Project (No. 2018SHZDZX01). The authors would like to thank anonymous reviewers for their helpful comments.

\bibliographystyle{abbrvnat}
\bibliography{Bandit}

\clearpage
\onecolumn
\appendix

\input{Appendix}

\end{document}

%% file: Introduction.tex
\section{INTRODUCTION}
\label{sec:introduction} 

We study stochastic multi-armed bandits \citep{LR:85,Auer:02,lattimore:19}, where the learning agent sequentially takes actions in order to maximize its cumulative reward. As the agent learns through experience, it faces a trade-off between exploration and exploitation: exploiting actions that maximize immediate rewards, as estimated by its current model; or improving its future rewards by exploring and learning a better model. Side information, such as the \emph{prior distribution} of arm means in \emph{Thompson sampling (TS)} \citep{T:33,CL:11,Agrawal12,AG:13,Russo-Van:14,AL:17}, can improve the statistical efficiency of the bandit algorithm and make it more practical.

\renewcommand{\thefootnote}{\fnsymbol{footnote}}
\footnotetext[1]{The work started while being at Google Research.}
\renewcommand{\thefootnote}{\arabic{footnote}}

While the prior is a blessing when correctly specified, a misspecified prior is a curse. Take online advertising as an example. It is well known that click probabilities of ads are low. Therefore, when estimating the click probability of a cold-start ad, it is important to model this structure. One approach would be Bayesian modeling, where the prior distribution is beta with a low mean. The shortcoming of this approach is that the prior needs to be specified, and is potentially misspecified. Therefore, design of bandit algorithms that depend less on exact priors is an important direction. 

To address this issue, we study \emph{random-effect models} \citep{Henderson:75,Robinson:91} in the bandit setting, and refer to the setting as a \emph{random-effect bandit}. Random-effect models were developed in statistics and econometrics \citep{Diggle:13,Wooldridge:01}, and are frequentist counterparts of hierarchical Bayesian models \citep{Carlin-Louis}. In our model, the \emph{arm means are sampled i.i.d.\ from a fixed unknown distribution}. The estimator of arm means is a weighted sum of two terms. The first term is the average of observed rewards of the arm. The second term estimates the common mean from all observations. The weights are chosen adaptively based on data, and balance the common mean estimate with that of the specific arm. Due to this structure, the resulting estimator of arm means is more statistically efficient than in the classical setting.

Our proposed bandit algorithm uses \emph{upper confidence bounds (UCBs)}, which is a popular approach to exploration with guarantees \citep{LR:85,Auer:02,Audibert09,Garivier:11}. In round $t \in [n]$, it pulls the arm with the highest UCB, observes its reward, and then updates its estimated arm means and their high-probability confidence intervals. The main difference from the classical algorithms is that all estimates are based on the random-effect model. Our method is essentially a random-effect \ucb \citep{Auer:02}, and thus we call it \reucb. 

Since our arm means are stochastic, \reucb is related to both TS and \bayesucb \citep{KCG:12}, which rely on posterior distributions. TS is popular in practice, but the assumption of knowing the prior exactly is rarely satisfied. In \reucb, we do not require that the prior is \emph{fully specified}, and thus we relax this assumption.

We make the following contributions. First, we introduce the assumption of random-effect models to multi-armed bandits, and properly formulate the corresponding bandit problem. Second, we propose a UCB-like algorithm for this problem, which we call \reucb. \reucb estimates arm means using the \emph{best linear unbiased predictor (BLUP)} \citep{Henderson:75,Robinson:91}, a method of estimating random effects without assumptions on distributions. The BLUP estimates leverage the structure of our problem and yield tighter confidence intervals than those of \ucb. Third, we analyze \reucb and derive an upper bound on its $n$-round Bayes regret \citep{Russo-Van:14} that reflects the structure of our problem. The main challenge in our regret analysis is the underspecified prior. Specifically, \reucb estimates the distribution of arm means from all observations and then uses it to estimate the mean of each arm. As a result, the estimated arm means are correlated, unlike in a typical multi-armed bandit. Finally, we evaluate \reucb empirically on a range of problems, such as Gaussian and Bernoulli bandits, and a movie recommendation problem. We observe that \reucb outperforms or is comparable to TS while using less prior knowledge.

%% file: Setting.tex
\section{RANDOM-EFFECT BANDITS}
\label{sec:setting}

We study a stochastic $K$-armed bandit \citep{LR:85,Auer:02,lattimore:19} where the number of arms can be large but finite. Because the mean rewards of some arms may not be reliably estimated due to many arms, it is challenging to explore all suboptimal arms efficiently. To overcome this challenge, we introduce a novel modeling assumption to multi-armed bandits.

We assume that the \emph{mean reward} of arm $k \in [K]$ follows a \emph{random-effect model} 
\begin{equation}
  \label{payoff-M1}
  \mu_k
  = \mu_0 + \delta_k\,,
\end{equation}
where $\mu_0$ is a common mean, $\delta_k \sim P^{(\mu)}(0, \sigma_0^2)$ is a random offset from that mean, and $P^{(\mu)}(0,\sigma_0^2)$ is a distribution with zero mean and variance $\sigma_0^2$. Thus $\mu_k$ is a random variable with mean $\mu_0$ and variance $\sigma_0^2$. With a lower variance, the differences among the arms are smaller. We improve over traditional bandit designs \citep{Auer:02} by using the stochasticity of $\mu_k$. Unlike in Thompson sampling \citep{T:33,CL:11,Russo-Van:14} or \bayesucb \citep{KCG:12}, we do not assume that the prior of arm means is \emph{conjugate} or \emph{fully specified}. We only require that $P^{(\mu)}(0,\sigma_0^2)$ has a finite second-order moment.

The reward of arm $k$ after the $j$-th pull is denoted by $r_{k,j}$ and we assume that it is generated i.i.d.\ as
\begin{equation}
  \label{payoff-R}
  r_{k,j}
  \sim P^{(r)}(\mu_k, \sigma^2)\,,
\end{equation}
where $P^{(r)}(\mu_k, \sigma^2)$ is a distribution with mean $\mu_k$ and variance $\sigma^2$. Similarly to $P^{(\mu)}(0, \sigma_0^2)$ in \eqref{payoff-M1}, we only require that its second-order moment is finite.

Our bandit has $K$ arms and a horizon of $n$ rounds. Before the first round, the mean reward of each arm is generated according to \eqref{payoff-M1}. In round $t \in [n]$, the agent pulls arm $I_t \in [K]$ and observes its stochastic reward, drawn according to \eqref{payoff-R}. For any arm $k$ and round $t$, we denote by $n_{k,t}$ the number of pulls of arm $k$ up to round $t$, and by $r_{k,1}, \dots, r_{k,n_{k,t}}$ the sequence of associated rewards. We call this problem a \emph{random-effect bandit}.

%% file: ModelEstimation.tex
\section{MODEL ESTIMATION}
\label{sec:model estimation}

This section describes our estimators of arms. In \cref{sec:estimation-1}, we estimate $\mu_k$ under the assumption that $\mu_0$ is known. In \cref{sec:estimation-2}, we provide an estimator for $\mu_k$ when $\mu_0$ is unknown. Additionally, we show how to estimate the variance parameters $\sigma_0^2$ and $\sigma^2$ in \cref{sec:estimation-3}. Because this section is devoted to estimating means and their variances at a fixed round $t$, we drop subindexing by $t$ to reduce clutter.

\subsection{Estimating $\mu_k$ When $\mu_0$ Is Known}
\label{sec:estimation-1}

We estimate $\mu_k$ using the \emph{best linear unbiased prediction (BLUP)}, which is a common method for estimating random effects \citep{Henderson:75,Robinson:91}. The BLUP estimator of $\mu_k$ minimizes the mean squared error among the class of linear unbiased estimators that do not depend on the distribution of model error.

We call the sample mean of arm $k$ its \emph{direct estimator}, and define it as $\bar{r}_{k} = n_k^{-1}\sum_{j=1}^{n_{k}}r_{k,j}$. From \eqref{payoff-R}, we get that $\text{Var}(\bar{r}_{k})=n_k^{-1}\sigma^2$. We improve upon this estimator with a class of linear unbiased estimators of form
$$\breve{\mu}_{k}:=\mu_0+a(\bar{r}_{k}-\mu_0)\,,$$
where $a\in \mathbb{R}$ is a to-be-optimized coefficient. Since $\mu_k$ is random rather than fixed, BLUP minimizes the mean squared error of $\breve{\mu}_{k}$ with respect to $\mu_k$, which is $\min_a\E{(\breve{\mu}_{k}-\mu_k)^2}$. Note that 
\begin{align}
\E{(\breve{\mu}_{k}-\mu_k)^2} 
& = \E{[a (\bar{r}_k  - \mu_k) + (1 - a) (\mu_0 - \mu_k)]^2}
\notag\\
& = a^2n_{k}^{-1}\sigma^2+(1-a)^2\sigma_0^2\,,
\label{mse-a}
\end{align} 
where the last equality is from \eqref{payoff-M1} and \eqref{payoff-R}, and that the reward noise is independent of $\delta_k$. When \eqref{mse-a} is minimized with respect to $a$, the optimal value of $a$ is
\begin{equation}\label{w-def}
w_{k}=\sigma_0^2/(\sigma_0^2+n_k^{-1}\sigma^2)=1/(1+n_k^{-1}\sigma^2/\sigma_0^2)\,.
\end{equation}
Thus, if $\mu_0$, $\sigma_0^2$, and $\sigma^2$ were known; and we plugged our derived $w_k$ into the definition of $\breve{\mu}_k$, we would get the following BLUP estimator of $\mu_k$
\begin{equation}\label{mu-estimator}
\tilde{\mu}_{k}=\mu_0+w_{k}(\bar{r}_{k}-\mu_0)
=(1-w_{k})\mu_0+w_{k}\bar{r}_{k}\,.
\end{equation}
From \eqref{w-def}, we have that $\sigma^{-2}\sigma_0^2\leq w_{k}<1$ for $n_k\geq 1$, and that $w_k \to 1$ as $n_k$ increases. We also have 
\begin{equation}\label{wto1w}
w_k n_k^{-1}\sigma^2 = (1 - w_k)\sigma_\mu^2\,.
\end{equation}
These properties are important in our analysis.

The estimator $\tilde{\mu}_{k}$ in \eqref{mu-estimator} is biased. The degree of this bias depends on both $n_{k}$ and $\sigma^2 / \sigma_0^2$ in \eqref{w-def}. If the arm has not been pulled enough, $w_k$ is low and $\tilde{\mu}_{k}$ is biased towards $\mu_0$. So we are not as aggressive in exploring as if $w_k = 1$. As the arm is pulled more, $w_k \to 1$ and the bias reduces to zero. When $\sigma_0^2$ decreases, the gaps among the arms decrease, and the effect of $\mu_0$ increases. Similarly, as $\sigma^2$ increases, the uncertainty in the direct estimator $\bar{r}_k$ increases, and so does the effect of $\mu_0$.

Now we set $a=w_k$ in \eqref{mse-a} and get
\begin{align}\label{MSE-term1}
\E{(\tilde{\mu}_{k}-\mu_k)^2}
   & = w_k^2n_k^{-1}\sigma^2+(1-w_k)^2\sigma_0^2
    = w_{k}n_k^{-1}\sigma^2\notag\\
    & =:\tilde{\tau}_k^2\,,
\end{align} 
where the last step is from \eqref{wto1w}. As $\text{Var}(\bar{r}_k)=n_k^{-1}\sigma^2$,  \eqref{MSE-term1} shows that $\tilde{\mu}_{k}$ is a better estimator of $\mu_k$ than $\bar{r}_k$, since $w_k<1$.

\subsection{Estimating $\mu_k$ When $\mu_0$ Is Unknown} 
\label{sec:estimation-2}

When $\sigma_0^2$ and $\sigma^2$ are known, the mean of arm means $\mu_0$ can be estimated by the generalized least squares estimator \citep{Rao:01}. That estimator is
\begin{equation}\label{r0-estimator}
\bar{r}_{0}=\left[\sum\nolimits_{k=1}^K(1-w_{k})n_{k}\right]^{-1} \sum\nolimits_{k=1}^K(1-w_{k})n_{k}\bar{r}_{k}
\end{equation}
and we derive it in \cref{sec:r0-estimator}. The estimator is more statistically efficient than the ordinary least squares because it weights the mean estimates of individual arms by their heteroscedasticity. Since $(1-w_{k})n_{k}=\sigma^2/(\sigma_0^2+n_k^{-1}\sigma^2)\rightarrow \sigma_0^{-2}\sigma^2$ as $n_k\rightarrow\infty$, we get $\bar{r}_{0}-K^{-1}\sum\nolimits_{k=1}^K\mu_k\rightarrow 0$ as $n_k\rightarrow\infty$ for all $k$. This means that $\bar{r}_0$ is a consistent estimator of $\mu_0$.

Now we plug the estimator $\bar{r}_{0}$ of $\mu_0$ into \eqref{mu-estimator} and get a \emph{synthetic estimator} of $\mu_k$,
\begin{equation}\label{blup}
\hat{\mu}_{k}=(1-w_{k})\bar{r}_{0}+w_{k}\bar{r}_{k}\,.
\end{equation}
The key point underlying the synthetic estimator is the weight $w_{k}$, which automatically balances variation among the arms and the uncertainty of $\bar{r}_{k}$. The variance of $\hat{\mu}_{k}$ is
\begin{align}\label{MSE}
\E{(\hat{\mu}_{k}-\mu_k)^2}
&=w_{k} n_k^{-1} \sigma^2 + \frac{(1-w_{k})^2}{\sum\nolimits_{k=1}^K n_{k} (1-w_{k})}  \sigma^2\notag\\
&=: \tau_{k}^2\,.
\end{align}
The derivation of \eqref{MSE} is in \cref{sec:derivation-mse}. The classical estimator of arm means in multi-armed bandits can be compared to that in random-effect bandits as follows.

\begin{proposition}\label{compare-tauANDmab}
For any arm $k \in [K]$, and any $\sigma^2 > 0$ and $n_k \geq 1$, we have $\tau_k^2< \sigma^2/n_k$.
\end{proposition}

The proof is in \cref{sec:proof-prop}. Proposition \ref{compare-tauANDmab} shows that $\tau_k^2$ is always lower than $\sigma^2 / n_k$ when $\sigma^2 > 0$, where the latter is the variance estimate in the classical bandit setting. In the worst case, for $\sigma^2 = 0$, we get $\tau_k^2= \sigma^2/n_k$, implying that the variance of $\hat{\mu}_{k}$ equals to that of $\bar{r}_k$. Thus, by using the synthetic estimator $\hat{\mu}_{k}$, we can be less optimistic than \ucb.

%% file: Algorithm.tex
\section{ALGORITHM}
\label{sec:algorithm}

\begin{algorithm}[!t]
\caption{\reucb for random-effect bandits.}
\label{alg:arm-ucb}
\begin{algorithmic}[1]
\FOR {$t = 1, \dots, n$}
\FOR {$k = 1, \dots, K$}
\STATE $U_{k,t} \gets \hat{\mu}_{k, t} + c_{k, t}$
\ENDFOR
\STATE \textbf{if} $t \leq K$ \textbf{then} $I_t \gets t$
\STATE \phantom{\textbf{if} $t \leq K$}
\textbf{else} $I_t \gets \arg\max_{k \in [K]} U_{k, t}$
\STATE Pull arm $I_t$ and observe its reward $r_{I_t, n_{I_t, t} + 1}$
\STATE Update all statistics
\ENDFOR
\end{algorithmic}
\end{algorithm}

We propose a UCB algorithm for random-effect bandits. The key idea in UCB algorithms \citep{Auer:02,Audibert09} is to pull the arm with the highest sum of its mean reward estimate and a weighted standard deviation of that estimate. In the setting of \cref{sec:estimation-2}, the estimated mean reward of arm $k$ is $\hat{\mu}_k$ in \eqref{blup} and its variance is $\tau_k^2$ in \eqref{MSE}. Due to space constraints, we do not present the algorithm for the setting in \cref{sec:estimation-1}. In this case, $\hat{\mu}_k$ would be replaced by $\tilde{\mu}_k$ and $\tau_k^2$ would be replaced by $\tilde{\tau}_k^2$.

Our algorithm is presented in \cref{alg:arm-ucb} and we call it \reucb, which stands for \emph{random-effect UCB}. We subindex all statistics in \cref{sec:model estimation} with an additional $t$, to make clear that we refer to round $t$. As an example, $\hat{\mu}_{k, t}$ and $\tau_{k, t}^2$ are the respective values of \eqref{blup} and \eqref{MSE} at the beginning of round $t$. \reucb works as follows. It is initialized by pulling each arm once. The \emph{upper confidence bound (UCB)} of arm $k$ in round $t$ is
\begin{align*}
  U_{k, t}
  = \hat{\mu}_{k, t} + c_{k, t}\,,
\end{align*}
where $c_{k, t}= \sqrt{a \tau_{k, t}^2 \log t}$
is its uncertainty bonus and $a > 0$ is a tunable parameter. In \cref{sec:analysis}, we prove regret bounds for $a \geq 1$. In round $t$, \reucb pulls the arm with the highest UCB $I_t = \arg\max_{k \in [K]} U_{k, t}$. To break ties, any fixed rule can be used.

\subsection{Related Algorithm Designs}

\reucb extends \ucb to a better BLUP estimator. For $w_{k, t} = 1$ and $a = 1$, \reucb has a similar UCB to \ucb, $U_{k, t} = \bar{r}_{k, t} + \sqrt{n_{k, t}^{-1} \sigma^2 \log t}$. We call this algorithm \reucbsinf and evaluate it empirically in \cref{fig:Gaussian-Others} in \cref{sec:add-exp}. Our results show that \reucbsinf is comparable to TS, but worse than \reucb. This shows the benefit of our model. Specifically, the estimate of $\mu_k$ in \reucb borrows information from other arms. This increases its statistical efficiency (\cref{compare-tauANDmab}), since the confidence interval of $\mu_k$ in \reucb can be narrower than in the classical setting. Note that \reucb with $a=1$ reduces to \ucb only if all weights $w_{k,t}$ are one. This could happen only if all arms were pulled infinitely often. So \reucb with $a=1$ does not behave like \ucb.

Due to assuming random arm means, \reucb is related to both TS \citep{T:33,CL:11,Russo-Van:14} and \bayesucb \citep{KCG:12}. Both \bayesucb and TS maintain posterior distributions. The computation of the posteriors requires that the mean of the prior $\mu_0$ is known. \reucb employs an alternative random-effect estimator that does not need it.

\citet{Li:11} proposed a hybrid model, where some coefficients are shared by all arms. However, this model is still traditional in the sense that the coefficients that are not shared are estimated separately in each arm. \citet{Gupta:21} recently proposed correlated multi-armed bandits, where the learning agent knows an upper bound on the mean reward of each arm given the mean reward of any other single arm. Such side information could be derived in our setting. However, it is also clearly not as powerful as using the observations of all arms jointly, as in \eqref{blup} and \eqref{MSE}.

%% file: Analysis.tex
\section{REGRET ANALYSIS}
\label{sec:analysis}

We derive an upper bound on the $n$-round regret of \reucb. In our setting, $\mu_k$ are random variables. Under the assumption that $r_{k, j}\sim \mathcal{N}(\mu_k, \sigma^2)$ and $\mu_k\sim\mathcal{N}(\mu_0, \sigma_0^2)$, which is used in one of our analyses, $\hat{\mu}_{k,t}$ is the \emph{maximum a posteriori (MAP)} estimate of $\mu_k$ given history, meaning that $\hat{\mu}_{k, t}$ can be viewed as a Bayesian estimator. Because of that, we adopt the Bayes regret \citep{Russo-Van:14} to analyze \reucb. The main novelty in our analysis is addressing the unknown mean of the prior.

Let $H_t=(I_\ell, r_{I_\ell, n_{I_\ell, \ell} + 1})_{\ell = 1}^{t - 1}$  be the \emph{history} at the beginning of round $t$ and $I_t$ be the pulled arm in round $t$. The regret is the difference between the rewards we would have obtained by pulling the optimal arm $I_* = \arg\max_{i \in [K]} \mu_i$ and the rewards that we did obtain in $n$ rounds. Our goal is to bound the Bayes regret $R_n = \E{\sum\nolimits_{t=1}^n \mu_{I_*}-\mu_{I_t}}$, where the expectation is over stochastic rewards and random $\mu_1, \dots, \mu_K$. Our main result is stated below.

\begin{theorem}\label{them-R-Bayes}
Consider a $K$-armed Gaussian bandit with rewards $r_{k, j} \sim \mathcal{N}(\mu_k, \sigma^2)$ and $\mu_k \sim \mathcal{N}(\mu_0, \sigma_0^2)$. Let \reucb use $\sigma_0^2$ and $\sigma^2$. Then (1) for any $a \geq 1$, the $n$-round Bayes regret of \reucb is
\begin{align*}
  R_n
  & \leq 2 \sqrt{
  \frac{a \log(1 + \sigma^{-2} \sigma_0^2 n)}{\log(1 + \sigma^{-2} \sigma_0^2)}
  \left(1 + \frac{\sigma^2}{K \sigma_0^2}\right) \sigma_0^2 K n \log n} \\
  &\quad +\frac{K\sigma_0^2+\sigma^2}{\sigma_0^2}\sqrt{\frac{8n\sigma_0^{2}\sigma^2}{\pi(\sigma_0^2+\sigma^2)}}.
\end{align*} 
(2) for any $a \geq 2$, the $n$-round Bayes regret of \reucb is obtained by replacing the last term above with $(1 + \log n)(K+\sigma^2\sigma_0^{-2})\sqrt{2 \sigma_0^2\sigma^2 / (\pi (\sigma^2 + \sigma_0^2))}$. 
\end{theorem}

\subsection{Discussion}
\label{sec:discussion}

Up to logarithmic factors, \cref{them-R-Bayes} shows that the $n$-round Bayes regret of \reucb is $O(K \sqrt{n})$ for $a \in [1, 2]$ and $O(\sqrt{K n})$ for $a \geq 2$. So the regret is sublinear in $n$ for any $a \geq 1$. Since both bounds increase in $a$, we suggest using $a = 1$, which performs extremely well in practice. Also note that the mean reward estimate in \eqref{blup} is a weighted sum of the estimate of $\mu_0$ (Term 1) and the per-arm reward mean (Term 2). The variance of the former is linear in $\sigma_0^2$, which gives rise to the linear dependence on $\sigma_0$ in Theorem \ref{them-R-Bayes}. Note that this dependence is standard in Bayes regret analyses \citep{Lu:19,Basu:21}, and it is due to using similar techniques in our proofs.

A Bayes regret lower bound exists for a $K$-armed bandit \citep{10.1214/aos/1176350495}. However, it has not been generalized to structured problems yet, including in seminal works on Bayes regret minimization \citep{Russo-Van:14}. Similarly, we also do not provide a matching lower bound in this work. Instead, we argue that our regret bound reflects the structure of our problem by comparing it to agents that use more information or less structure.

\cref{them-R-Bayes} is proved under the assumption that \reucb estimates $\mu_0$. Now consider a variant of \reucb where $\mu_0$ is known. This agent with more information can be analyzed similarly to \reucb. In this analysis, $\hat{\mu}_{k, t}$ and $\tau_{k, t}^2$ would be replaced by $\tilde{\mu}_{k, t}$ in \eqref{mu-estimator} and $\tilde{\tau}_{k, t}^2$ in \eqref{MSE-term1}, respectively. The resulting regret bound would be the same as in \cref{them-R-Bayes}, except for the extra factor of $1 + \sigma^2 / (K \sigma_0^2)$. Therefore, this factor can be viewed as the price for learning $\mu_0$. As it is $O(1 + 1 / K)$, its impact on the Bayes regret of \reucb is small when $K$ is large.

Now suppose that $\mu_0 \sim \mathcal{N}(0, \sigma_q^2)$. However, the structure that $\mu_0$ is the same for all arms is not modeled. This problem is equivalent to a Bayesian bandit with a per-arm prior $\mathcal{N}(0, \sigma_q^2 + \sigma_0^2)$ and \reucb with known $\mu_0 = 0$ can solve it. When analyzed, the leading term in \cref{them-R-Bayes} would be
\begin{align*}
  2 \sqrt{\frac{a \log(1 + \sigma^{-2} (\sigma_q^2 + \sigma_0^2) n)}
  {\log(1 + \sigma^{-2} (\sigma_q^2 + \sigma_0^2))}
  (\sigma_q^2 +\sigma_0^2) K n \log n}\,.
\end{align*}
Thus, up to logarithmic factors, our regret bound is lower whenever $(1 + \sigma^2 / (K \sigma_0^2)) \sigma_0^2 \leq \sigma_q^2 + \sigma_0^2$, and it is beneficial to learn the common $\mu_0$ in this case. For any $\sigma_q > 0$, this is guaranteed as $K$ increases.

\cref{them-R-Bayes} can be extended in several ways. First, we generalize the model in \eqref{payoff-R} to arm-dependent reward noise. Specifically, the reward of arm $k$ after the $j$-th pull is drawn i.i.d.\ as $r_{k, j} \sim \mathcal{N}(\mu_k, \sigma_k^2)$, where the variance $\sigma_k^2$ may depend on $k$. In \cref{sec:varying}, we show that the Bayes regret bound in \cref{them-R-Bayes} still holds for $\sigma^2 = \max_{k \in [K]} \sigma_k^2$. Second, the Gaussian assumption in \cref{them-R-Bayes} is replaced with bounded sub-Gaussianity in \cref{sec:extensions}.

Finally, we would like to point out the limitations of our results. First, our proofs rely on well-behaved posterior distributions, either Gaussian or bounded sub-Gaussian. This is due to limitations of existing Bayes regret analyses, which use it to bound tail events conditioned on history \citep{Russo-Van:14}. We observe that it is not needed for good practical performance and believe that better analyses will be possible in the future. Second, our proofs are under the assumption that $\sigma_0^2$ and $\sigma^2$ are known. This is akin to existing Bayes regret analyses. We experiment with estimating $\sigma_0^2$ and $\sigma^2$ in \cref{sec:experiments}.

Now we are ready to prove \cref{them-R-Bayes}.

\subsection{Proof of \cref{them-R-Bayes}}
\label{sec:proof}

Let the confidence interval of arm $k$ in round $t$ be
\begin{equation}\label{confidence-C}
c_{k, t}
= \sqrt{a\tau_{k, t}^2 \log t}
= \sqrt{2 \tau_{k, t}^2 \log(1 / \delta_t)}\,, 
\end{equation} 
where $\delta_t = t^{- a / 2}$. 
Define the events that all confidence intervals in round $t$ hold as
\begin{align*}
E_{R;t}& =\left\{\forall k \in [K]: \hat{\mu}_{k,t} - \mu_{k} \leq c_{k,t}\right\}\,, \notag\\
E_{L;t}& =\left\{\forall k \in [K]: \mu_{k} - \hat{\mu}_{k,t} \leq c_{k,t}\right\}\,.
\end{align*} 
Fix round $t$. The regret in round $t$ is decomposed as
\begin{align}\label{regret-decom}
\E{\mu_{I_*}-\mu_{I_t}} 
& = \E{\condE{\mu_{I_*}-\mu_{I_t}}{H_t}}\notag\\
& = \E{\condE{\mu_{I_*}-\hat{\mu}_{I_*,t}-c_{I_*,t}}{H_t}}\notag\\
&\quad +\E{\condE{\hat{\mu}_{I_*,t}+c_{I_*,t}-\mu_{I_t}}{H_t}}\notag\\
& \leq \E{\condE{\mu_{I_*}-\hat{\mu}_{I_*,t}-c_{I_*,t}}{H_t}}\notag\\
&\quad +\E{\condE{\hat{\mu}_{I_t,t}+c_{I_t,t}-\mu_{I_t}}{H_t}}\,.
\end{align} 
The first equality is by the tower rule. The second is from the fact $\hat{\mu}_{k,t}$ and $c_{k,t}$ are deterministic given $H_t$. The inequality is from the fact that $I_t$ maximizes $\hat{\mu}_{k,t}+c_{k,t}$ over $k\in [K]$ given $H_t$. For each term in \eqref{regret-decom}, we get 
\begin{align*}
& \condE{\mu_{I_*}-\hat{\mu}_{I_*,t}-c_{I_*,t}}{H_t} \\
& \quad \leq \condE{(\mu_{I_*}-\hat{\mu}_{I_*,t})\mathds{1}\{\bar{E}_{L;t}\}}{H_t} ,\notag\\
& \condE{\hat{\mu}_{I_t,t}+c_{I_t,t}-\mu_{I_t}}{H_t} \\
& \quad \leq 2 \condE{c_{I_t,t}}{H_t}+\condE{(\hat{\mu}_{I_t,t}-\mu_{I_t})\mathds{1}\{\bar{E}_{R;t}\}}{H_t},
\end{align*}
where the inequalities are from the fact that $\mu_{I_*}-\hat{\mu}_{I_*,t}\leq c_{I_*,t}$ on $E_{L;t}$, and that $\hat{\mu}_{I_t,t}-\mu_{I_t}\leq c_{I_t,t}$ on $E_{R;t}$.
By chaining all inequalities, the regret is bounded as 
\begin{align}\label{regret-decom-upper}
&\E{\sum\nolimits_{t=1}^n(\mu_{I_*}-\mu_{I_t})} 
 \leq  2\E{\sum\nolimits_{t=1}^nc_{I_t,t}}\notag\\
&\quad +\E{\sum\nolimits_{t=1}^n\condE{(\hat{\mu}_{I_t,t}-\mu_{I_t})\mathds{1}\{\bar{E}_{R;t}\}}{H_t}}\notag\\
& \quad +\E{\sum\nolimits_{t=1}^n\condE{(\mu_{I_*}-\hat{\mu}_{I_*,t})\mathds{1}\{\bar{E}_{L;t}\}}{H_t}}.
\end{align}

We start with the first term in \eqref{regret-decom-upper}. This term depends on $\tau_{k,t}^2$, which depends on the pulls of all arms, as defined in \eqref{MSE}. Therefore, it is challenging to analyze. To do that, we use \cref{sec:tau} of \cref{sec:lemma}, which shows that
\begin{align*}
  \tau_{k, t}^2
  \leq \beta \tilde{\tau}_{k, t}^2
  = \beta \frac{1}{\sigma_0^{-2} + \sigma^{-2} n_{k, t}}
\end{align*}
for $\beta = 1 + \sigma^2 / (K \sigma_0^2)$. This means that we can bound $\tau_{k,t}^2$ by only considering arm $k$. Then
\begin{align*}
  \sum_{t = 1}^n c_{I_t, t}
  & \leq \sum_{t = 1}^n \sqrt{a \tau_{I_t, t}^2 \log n} \\
  & \leq \sqrt{a \beta n \log n}
  \sqrt{\sum_{t = 1}^n \frac{1}{\sigma_0^{-2} + \sigma^{-2} n_{I_t, t}}}\,,
\end{align*}
where the first inequality is from the definition of $c_{k, t}$ and $\log t \leq \log n$, and we used the Cauchy-Schwarz inequality in the second one.

Note that $x / \log(1 + x) \leq m / \log(1 + m)$ for $x \in [0, m]$, because $x = \log(1 + x)$ at $x = 0$ and $x$ grows faster than $\log(1 + x)$ on $[0, m]$. Now we apply this bound for $x = \sigma^{-2} / (\sigma_0^{-2} + \sigma^{-2} n_{I_t, t})$ and $m = \sigma^{-2} \sigma_0^2$, and get
\begin{align*}
  \frac{1}{\sigma_0^{-2} + \sigma^{-2} n_{I_t, t}}
  & \leq \gamma \log\left(1 +
  \frac{\sigma^{-2}}{\sigma_0^{-2} + \sigma^{-2} n_{I_t, t}}\right) \\
  & = \gamma \log\frac{\sigma_0^{-2} + \sigma^{-2} (n_{I_t, t} + 1)}
  {\sigma_0^{-2} + \sigma^{-2} n_{I_t, t}}\,,
\end{align*}
where $\gamma = \sigma_0^2 / \log(1 + \sigma^{-2} \sigma_0^2)$. The above leads to telescoping and
\begin{align*}
  \sum_{t = 1}^n \frac{1}{\sigma_0^{-2} + \sigma^{-2} n_{I_t, t}}
  & \leq \gamma K \left[\log(\sigma_0^{-2} + \sigma^{-2} n) -
  \log(\sigma_0^{-2})\right]\notag\\
  & = \gamma K \log(1 + \sigma^{-2} \sigma_0^2 n)\,,
\end{align*}
where we used that any arm is pulled at most $n$ times. Now we put everything together and get
\begin{align}\label{regret-decom-upper1}
  \sum_{t = 1}^n c_{I_t, t}
  \leq \sqrt{\frac{a \log(1 + \sigma_0^2 n)}{\log(1 + \sigma^{-2}\sigma_0^2)}
  \beta \sigma^{-2}\sigma_0^2 K n \log n}\,.
\end{align}
The next step is the second term in \eqref{regret-decom-upper}. To bound it, we show in \cref{sec:mu0Dist} of \cref{sec:lemma} 
that $\mu_{k} | H_t \sim \mathcal{N}(\hat{\mu}_{k, t}, \tau_{k, t}^2)$, under the assumption of $r_{k, j}\sim \mathcal{N}(\mu_k, \sigma^2)$ and $\mu_k\sim\mathcal{N}(\mu_0, \sigma_0^2)$. By using this property,
\begin{align*}
&\quad \condE{(\hat{\mu}_{I_t,t}-\mu_{I_t})\mathds{1}\{\bar{E}_{R;t}\}}{H_t} \notag\\
\leq & \sum_{k=1}^K\frac{1}{\sqrt{2\pi\tau_{k,t}^{2}}}
\int\limits_{x\geq c_{k,t}}x\exp\left(-\frac{x^2}{2\tau_{k,t}^{2}}\right)dx\leq\frac{\delta_t}{\sqrt{2\pi}}\sum_{k=1}^K\tau_{k,t}\,,
\end{align*}
where the last inequality is from \eqref{confidence-C}. It follows that for $a \geq 1$,
\begin{align}\label{regret-decom-upper-term2}
& \E{\sum\limits_{t=1}^n\condE{(\hat{\mu}_{I_t,t}-\mu_{I_t})\mathds{1}\{\bar{E}_{R;t}\}}{H_t}}\notag\\ 
& \quad \leq \frac{1}{\sqrt{2\pi}}\sum\limits_{t=1}^nt^{-1/2} \beta K \sqrt{\frac{\sigma^2}{1+\sigma^2\sigma_0^{-2}}}\notag\\
& \quad \leq \frac{1}{\sqrt{\pi}} \beta K \sqrt{\frac{2n\sigma^2}{1+\sigma^2\sigma_0^{-2}}}\,,
\end{align} 
where the first inequality follows from $\delta_t\leq t^{-1/2}$ for $a \geq 1$ (\cref{sec:tau} of \cref{sec:lemma}), and the last one is from $\sum_{t=1}^nt^{-1/2}\leq 2\sqrt{n}$. Similarly, when $a \geq 2$,
\begin{align}\label{regret-decom-upper-term22}
& \E{\sum\limits_{t=1}^n\condE{(\hat{\mu}_{I_t,t}-\mu_{I_t})\mathds{1}\{\bar{E}_{R;t}\}}{H_t}}\notag\\
 & \quad \leq \frac{\beta K (1 + \log n)}{\sqrt{2 \pi}} \sqrt{\frac{\sigma^2}{1+\sigma^2\sigma_0^{-2}}}\,,
\end{align} 
where the inequality is from $\delta_t \leq t^{- 1}$ for $a \geq 2$ and $\sum_{t=1}^nt^{-1}\leq 1+\log n$. 

At last, we study the third term in \eqref{regret-decom-upper}. The result in \eqref{regret-decom-upper-term2} or \eqref{regret-decom-upper-term22} holds for the term. 
Therefore, by combing \eqref{regret-decom-upper}, \eqref{regret-decom-upper1},  \eqref{regret-decom-upper-term2}, and \eqref{regret-decom-upper-term22},
the theorem is proved.



%% file: Experiments.tex
\section{SYNTHETIC EXPERIMENTS}
\label{sec:experiments}

\begin{figure*}[!ht]
\centering
\begin{subfigure}[b]{0.67\columnwidth}
\includegraphics[keepaspectratio,width=1\linewidth]{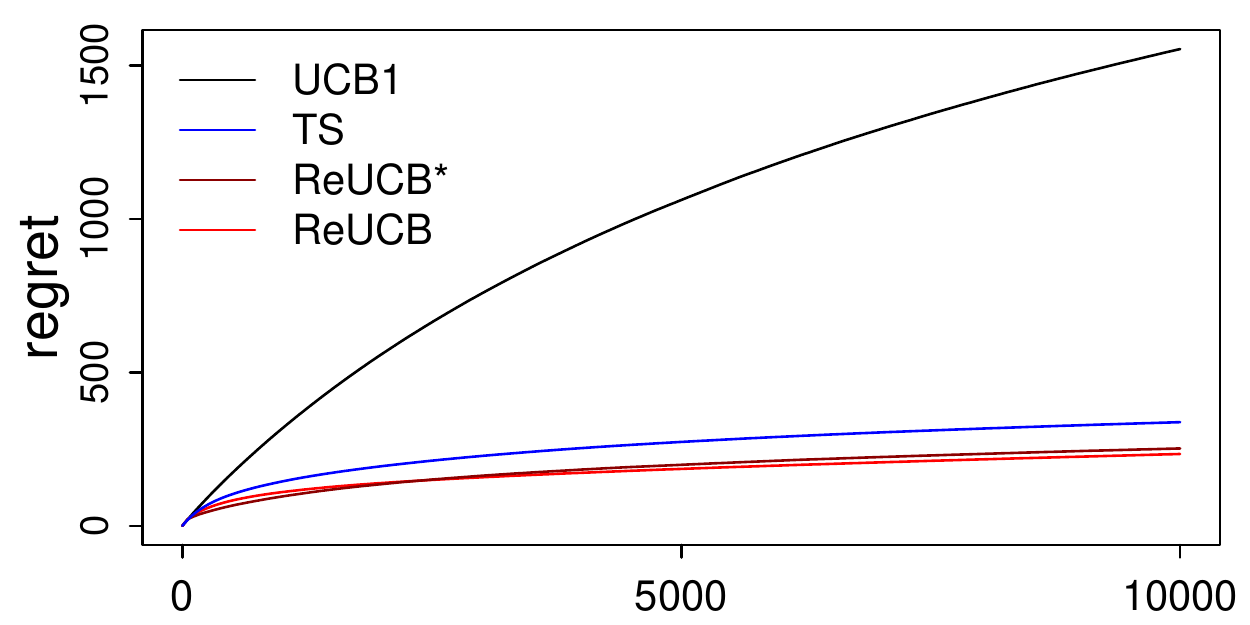}
\includegraphics[keepaspectratio,width=1\linewidth]{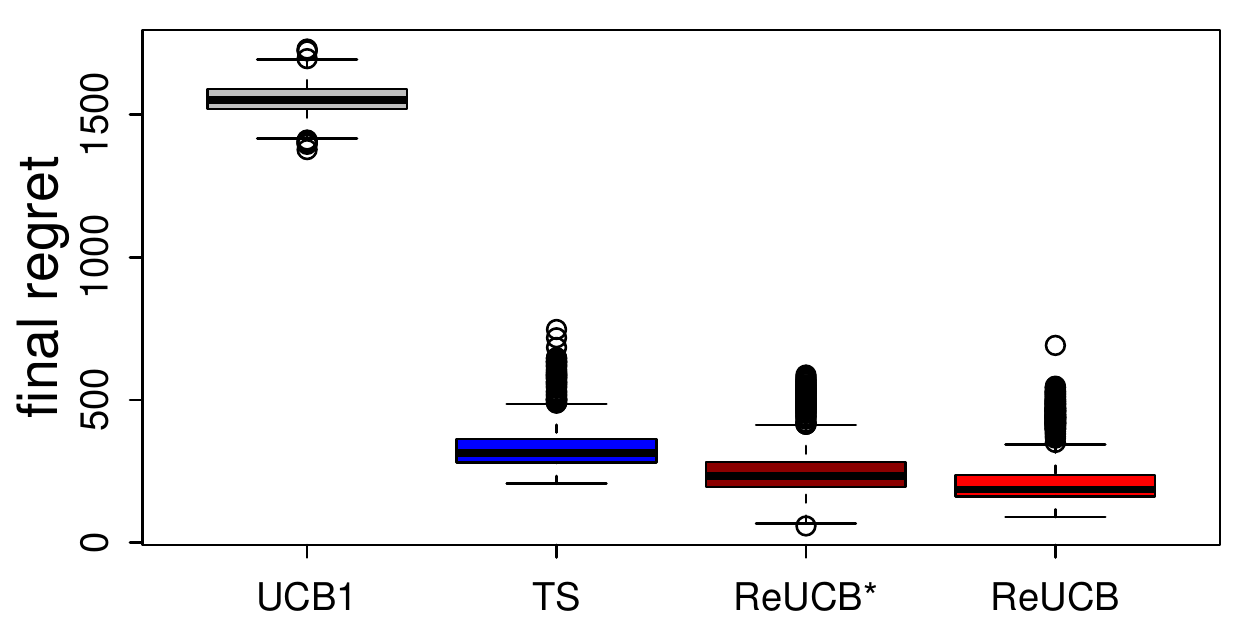}
\vspace{-0cm}
\subcaption{$\mu_k\sim \mathcal{N}(1,0.04)$ }
\end{subfigure}
\begin{subfigure}[b]{0.67\columnwidth}
\includegraphics[keepaspectratio,width=1\linewidth]{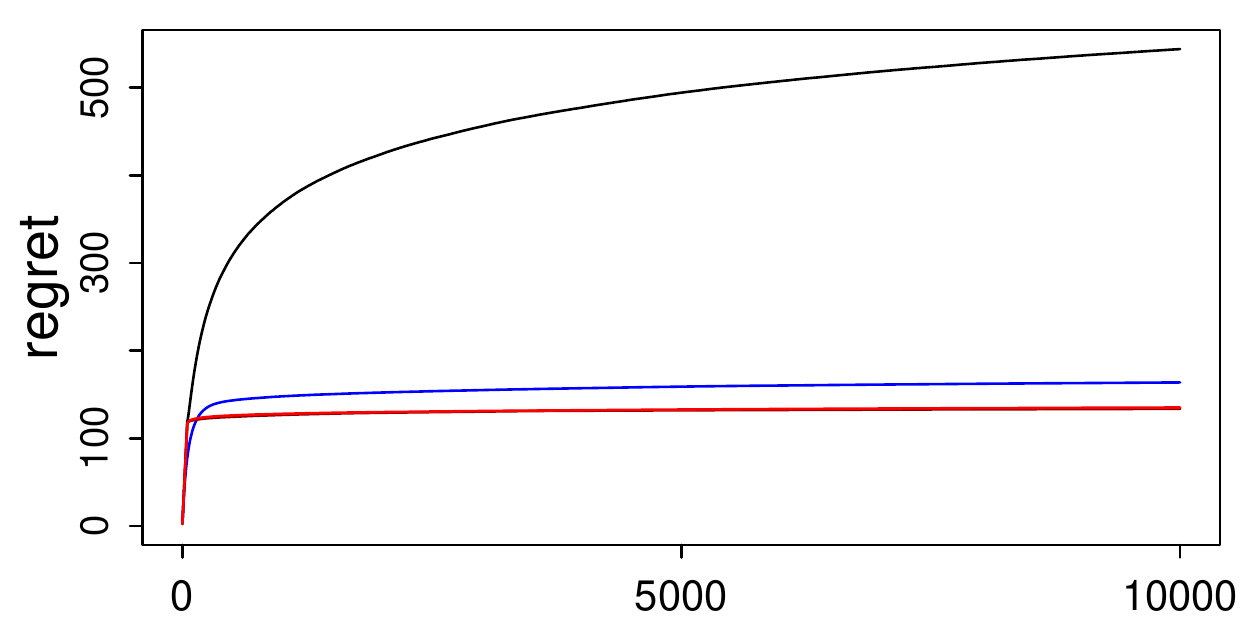}
\includegraphics[keepaspectratio,width=1\linewidth]{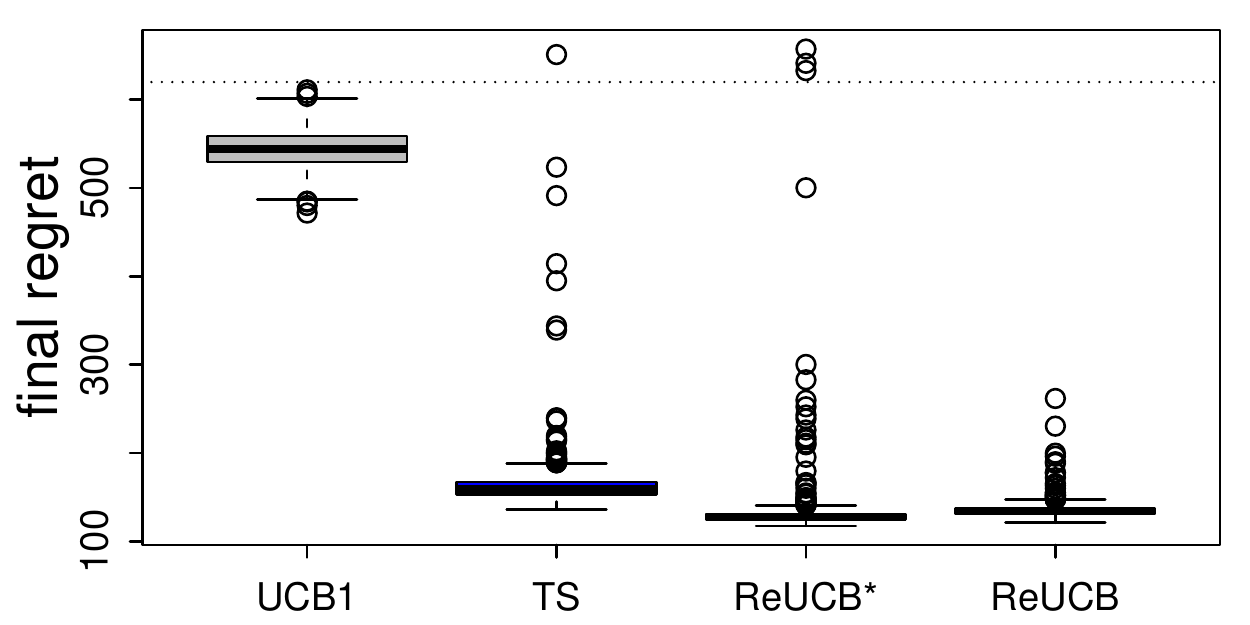}
\vspace{-0cm}
\subcaption{$\mu_k\sim \mathcal{N}(1,1)$}
\end{subfigure}
\begin{subfigure}[b]{0.67\columnwidth}
\includegraphics[keepaspectratio,width=1\linewidth]{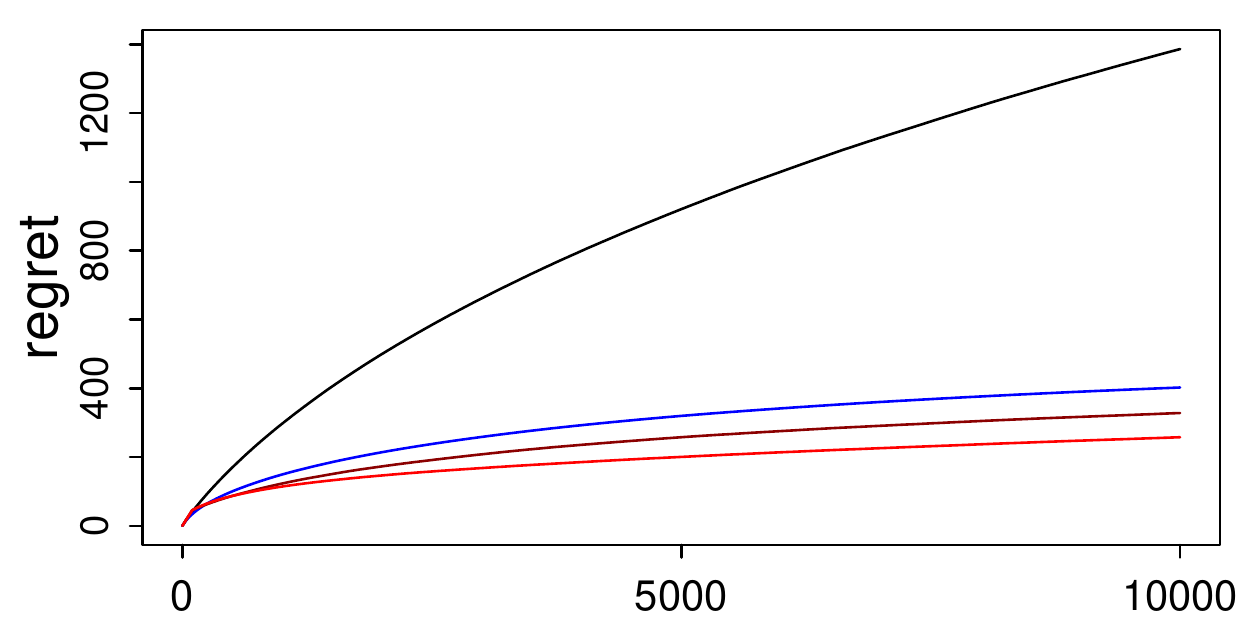}
\includegraphics[keepaspectratio,width=1\linewidth]{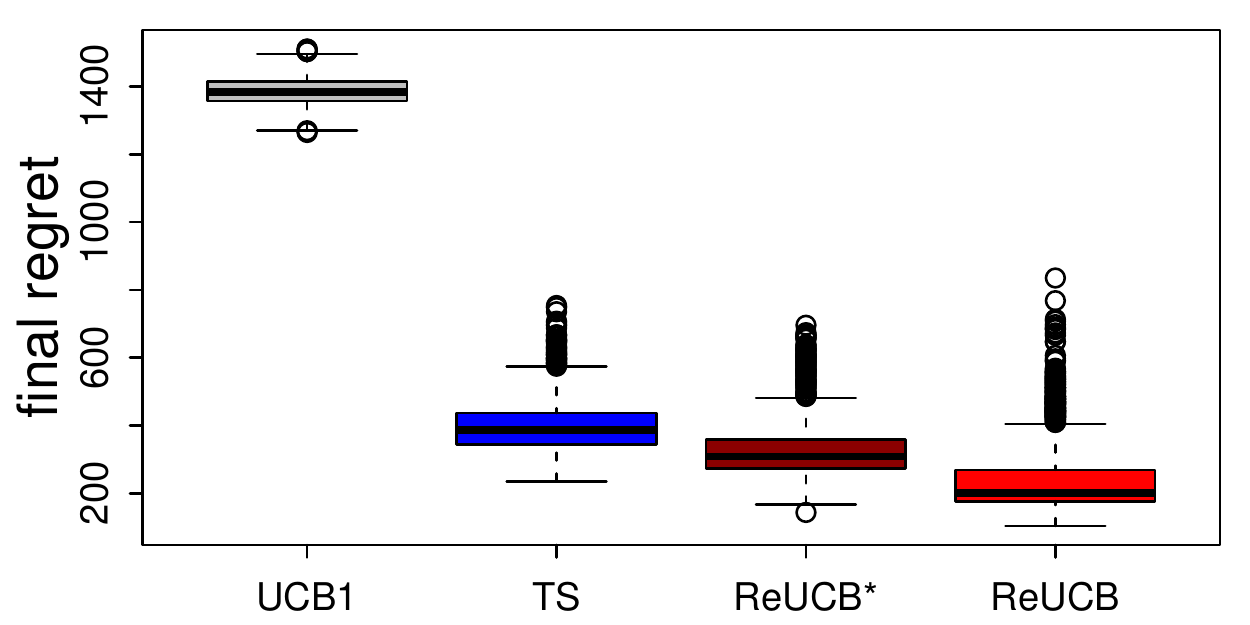}
\vspace{-0cm}
\subcaption{$\mu_k\sim \mathcal{U}[1,2]$}
\end{subfigure}
\caption{$K$-armed Gaussian bandits with $\mu_k$ from $\mathcal{N}(\mu_0, \sigma_0^2)$ and $\mathcal{U}[1, 2]$. Upper row: Regret as a function of round $n$. Lower row: Distribution of the regret at the final round.} 
\label{fig:Gaussian-armed}
\vspace{-0in}
\end{figure*}

We study two bandit settings: Gaussian (\cref{sec:gaussian bandits}) and Bernoulli (\cref{sec:bernoulli bandits}). Moreover, in \cref{sec:mis bandits}, we study misspecified priors. \reucb is compared to \ucb \citep{Auer:02} and \ts \citep{T:33}. \ts is chosen because it uses the same structure as \reucb, that arm means are random. However, it needs more knowledge, the prior distribution of $\mu_0$. Since \reucb is a UCB algorithm, it is natural to compare it to other UCB algorithms. We focus on \ucb due to its simplicity and popularity, but also compare to \bayesucb \citep{KCG:12} and \klucb \citep{Garivier:11} in \cref{fig:Gaussian-Others} of \cref{sec:add-exp}. Both \bayesucb and \klucb improve over \ucb, but are not better than \ts. This is consistent with other reported results in the literature \citep{Kveton:19}. There are many other potential baselines, such as Giro \citep{Kveton:18} and PHE \citep{Kveton:19}. Our \reucb is fundamentally different from these methods, since our arm means are random. Also, when compared to these methods, \ts is typically a strong baseline \citep{Kveton:19}. Therefore, to make our empirical studies clean and focused, we compare to \ucb and \ts.

We evaluate two variants of \reucb: (1) \reucbs, where $\mu_0$ is estimated, and $\sigma_0^2$ and $\sigma^2$ are known; and (2) \reucb, where all of $\mu_0$, $\sigma_0^2$, and $\sigma^2$ are estimated. The variance estimators $\hat{\sigma}_{0,t}^2$ and $\hat{\sigma}_{t}^2$ are provided in \eqref{variance1} and \eqref{variance2}, respectively, of  \cref{sec:estimation-3}. Unless specified, the default priors in Gaussian and Bernoulli \ts are $\mathcal{N}(\mu_0, \sigma_0^2)$ and $\mathrm{Beta}(1,1)$, respectively. The upper confidence bound in \ucb is $\bar{r}_{k, t} + \sqrt{8 n_{k, t}^{-1} \sigma^2 \log t}$. This is a generalization of the original algorithm to $\sigma^2$-sub-Gaussian rewards. In the original algorithm, $\sigma^2 = 1 / 4$. In Gaussian bandits, we set $\sigma$ to Gaussian noise. In Bernoulli bandits, this reduces to the \ucb index since $\sigma^2 = 1 / 4$. All simulations are averaged over $1000$ independent runs.

\subsection{Gaussian Bandits}
\label{sec:gaussian bandits}

Our first experiment is on $K$-armed Gaussian bandits. The reward distribution of arm $k$ is $\mathcal{N}(\mu_k, \sigma^2)$ where $\sigma = 0.5$. We generate $\mu_k$ in two ways. First, $\mu_k$ are drawn independently from Gaussian prior $\mathcal{N}(\mu_0, \sigma_0^2)$, where we study two settings of $(\mu_0, \sigma_0^2)$: $(1, 0.04)$ (low coefficient of variation $0.2$) and $(1, 1)$ (high coefficient of variation $1$). Second, $\mu_k$ are drawn from uniform distribution $\mathcal{U}[1, 2]$. The number of arms is $K = 50$. The horizon is $n = 10^4$ rounds.

Figures \ref{fig:Gaussian-armed}(a) and \ref{fig:Gaussian-armed}(b) report results for Gaussian priors, while \cref{fig:Gaussian-armed}(c) shows results for the uniform prior. We observe that \ts works well and outperforms \ucb. \reucb has a much lower regret than both \ucb and \ts. Besides good average performance, the distribution of the regret in the final round (lower row in \cref{fig:Gaussian-armed}) shows good stability. The good performance of \reucb in \cref{fig:Gaussian-armed}(c) indicates that \reucb works for various priors. \reucb performs well empirically because our high-probability confidence intervals are narrower than in the classical setting (\cref{compare-tauANDmab}). It outperforms \ts with more information in Gaussian bandits because its confidence interval widths $\tau_{k, t}$ are narrower than the posterior widths of \ts.

Now we compare the regret of \reucbs and \reucb in \cref{fig:Gaussian-armed}. Clearly the estimation of $\sigma^2$ and $\sigma_0^2$ does not have a major impact on the regret of \reucb. In fact, the regret slightly decreases. We believe that this is due to the additional randomness in our method-of-moments estimators of $\sigma^2$ and $\sigma_0^2$. These results suggest that one limitation of our analysis, that $\sigma^2$ and $\sigma_0^2$ are known, is not a limitation in practice.

Finally, we report the run times of all algorithms. All experiments are conducted in R, on a PC with $3$GHz Intel i7 CPU, $8$GB RAM, and OS X operating system. In \cref{fig:Gaussian-armed}(a), a single run of \reucb, \ucb, and \ts takes on average $1.27$s, $1.16$s, and $3.19$s, respectively. So \reucb is slightly slower than \ucb but much faster than \ts, which is slower due to posterior sampling.

\subsection{Bernoulli Bandits}
\label{sec:bernoulli bandits}

\begin{figure*}[!ht]
\centering
\begin{subfigure}[b]{0.67\columnwidth}
\includegraphics[keepaspectratio,width=1\linewidth]{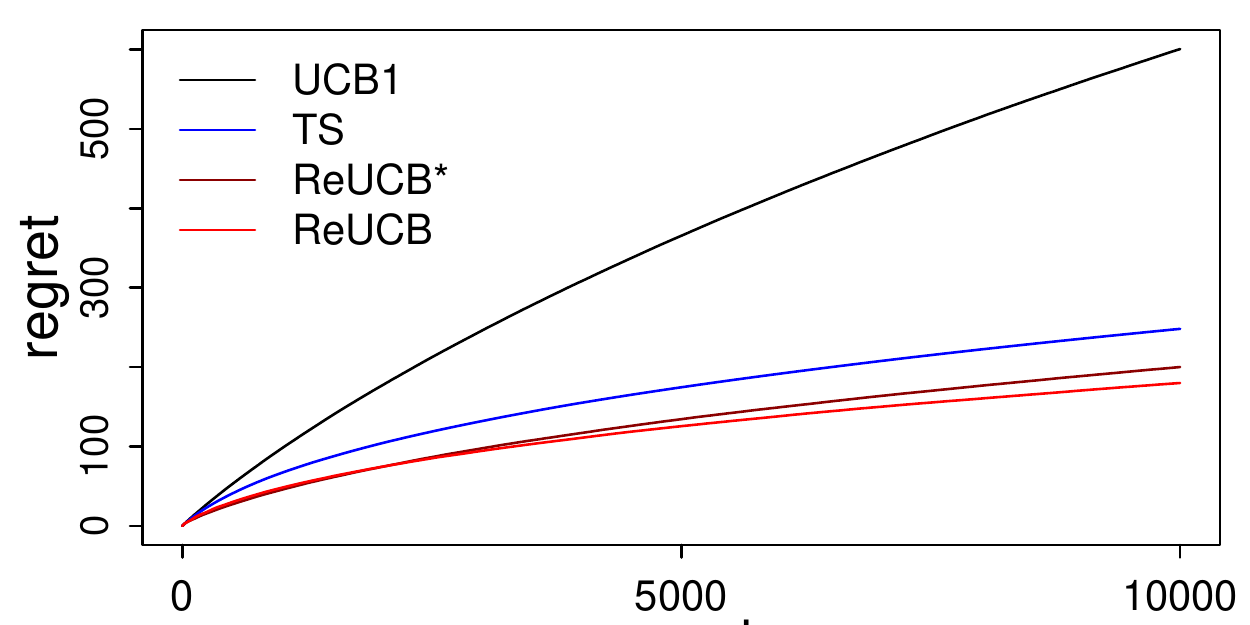}
\includegraphics[keepaspectratio,width=1\linewidth]{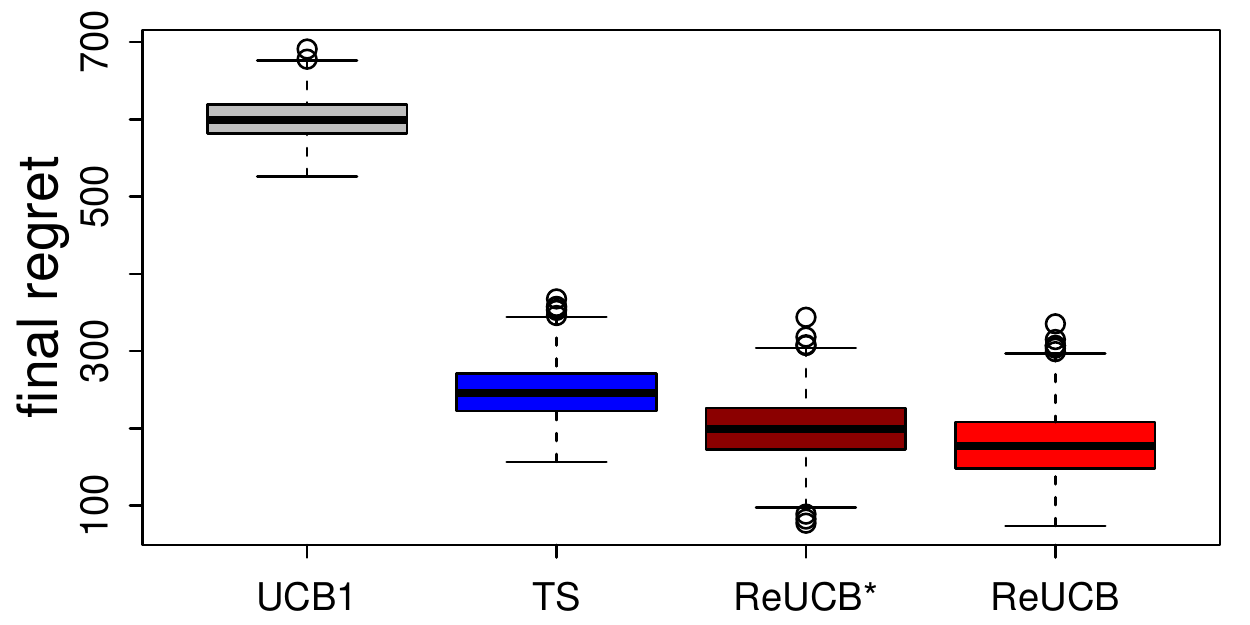}
\vspace{-0cm}
\subcaption{$K=20$}
\end{subfigure}
\begin{subfigure}[b]{0.67\columnwidth}
\includegraphics[keepaspectratio,width=1\linewidth]{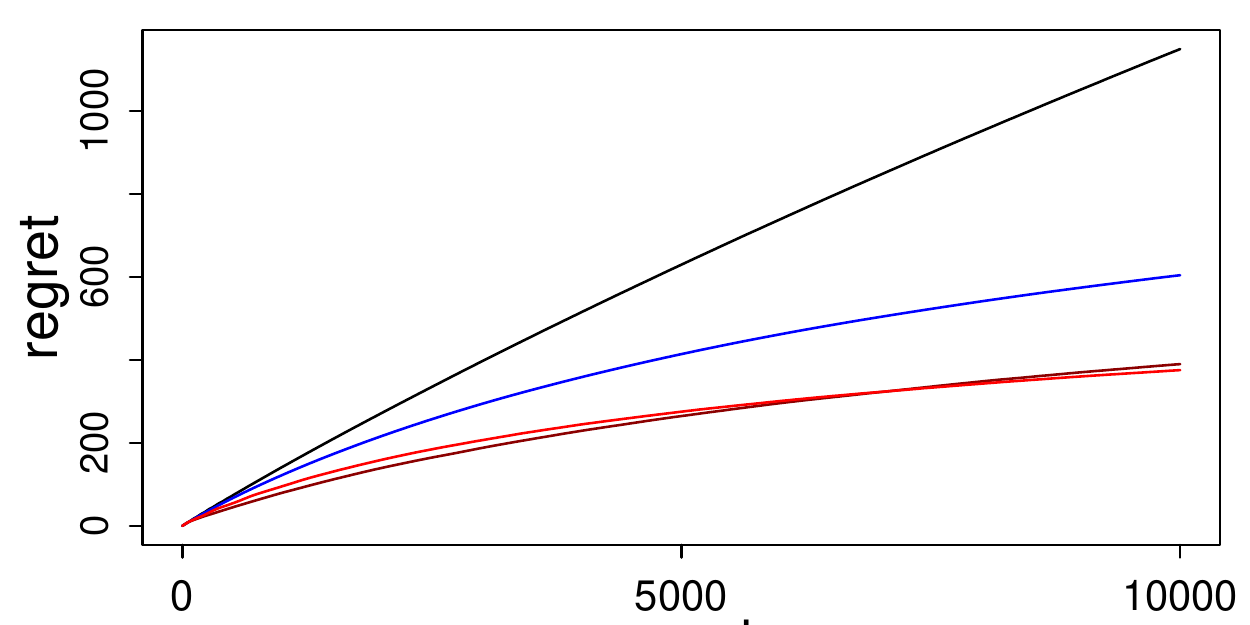}
\includegraphics[keepaspectratio,width=1\linewidth]{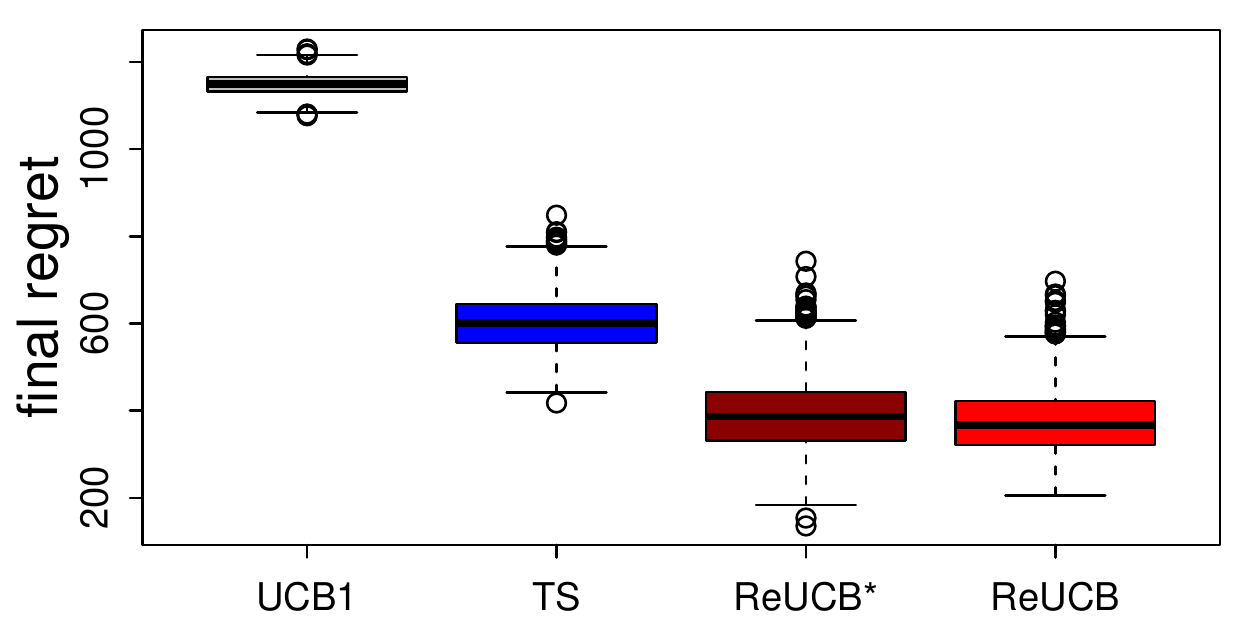}
\vspace{-0cm}
\subcaption{$K=50$}
\end{subfigure}
\begin{subfigure}[b]{0.67\columnwidth}
\includegraphics[keepaspectratio,width=1\linewidth]{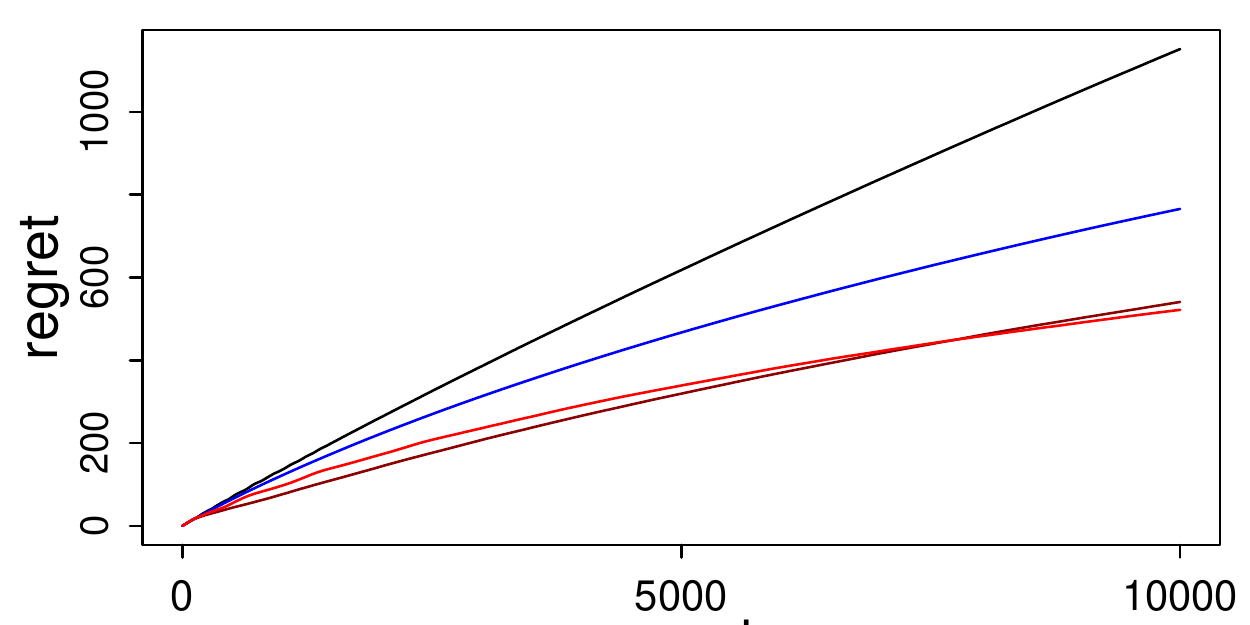}
\includegraphics[keepaspectratio,width=1\linewidth]{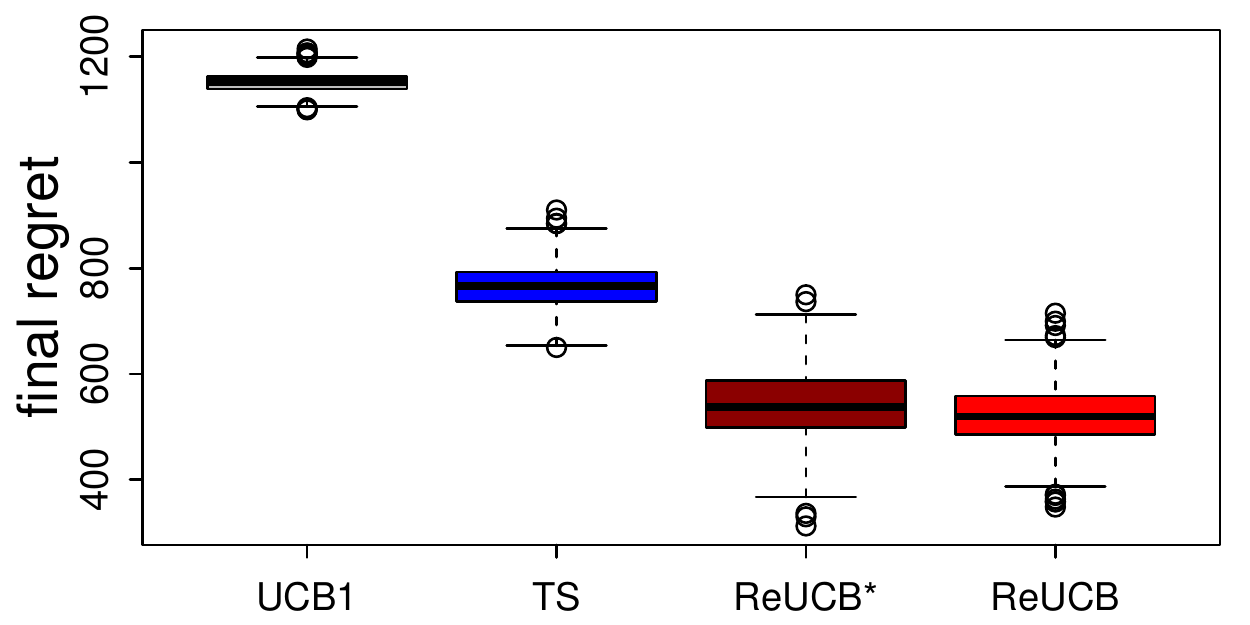}
\vspace{-0cm}
\subcaption{$K=100$}
\end{subfigure}
\caption{$K$-armed Bernoulli bandits with $\mu_k\sim \mathcal{U}[0.2, 0.5]$. Upper row: Regret as a function of round $n$. Lower row: Distribution of the regret at the final round.}
\vspace{-0in}
\label{fig:bernoulli}
\end{figure*}

The second experiment is conducted on $K$-armed Bernoulli bandits, where the reward distribution of arm $k$ is $\text{Bern}(\mu_k)$. The arm means $\mu_k$ are drawn i.i.d.\ from uniform distribution $\mathcal{U}[0.2, 0.5]$. We experiment with three settings for the number of arms $K \in \{20, 50, 100\}$, to show that \reucb performs well across all of them. The horizon is $n = 10^4$ rounds.

As in \cref{fig:Gaussian-armed}, we observe in \cref{fig:bernoulli} that \reucb has a much lower regret than \ucb and \ts, and performs similarly to \reucbs. To implement \reucbs, we set the maximum variance to $\sigma^2 = 1 / 4$, as suggested in \cref{sec:extensions}. Different from \cref{fig:Gaussian-armed}, \cref{fig:bernoulli} shows the regret for various $K$. As the number of arms $K$ increases, the gap between our approaches and the baselines increases.

\subsection{Model Misspecification}
\label{sec:mis bandits}

\begin{figure*}[!ht]
\centering
\begin{subfigure}[b]{0.67\columnwidth}
\includegraphics[keepaspectratio,width=1\linewidth]{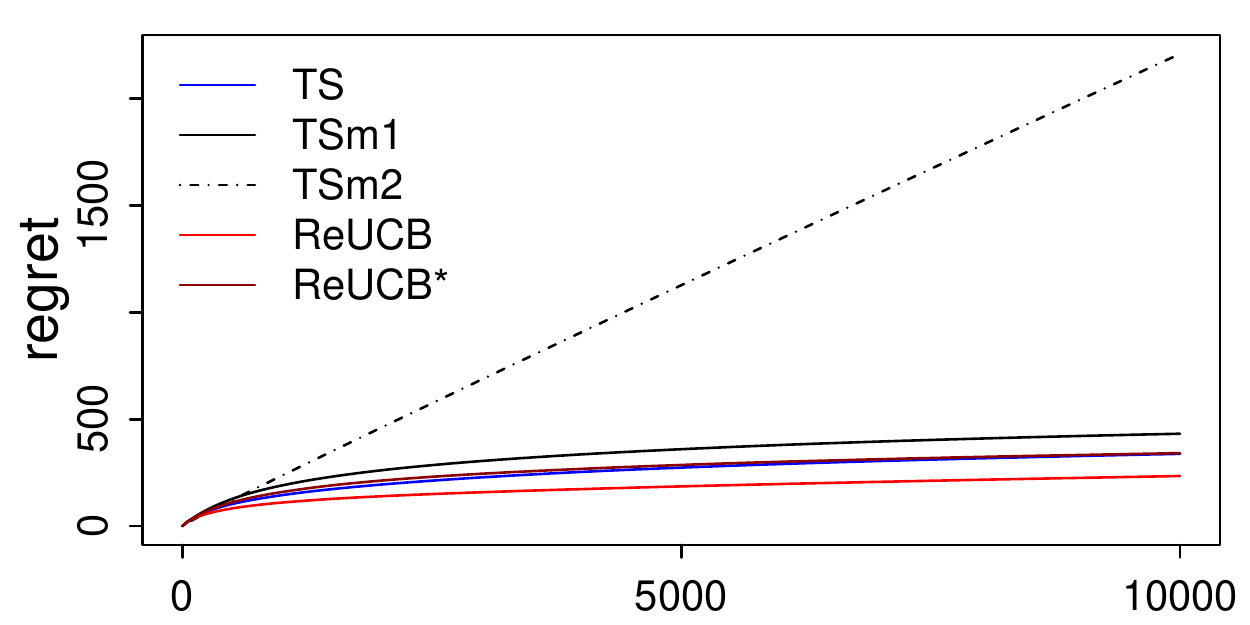}
\includegraphics[keepaspectratio,width=1\linewidth]{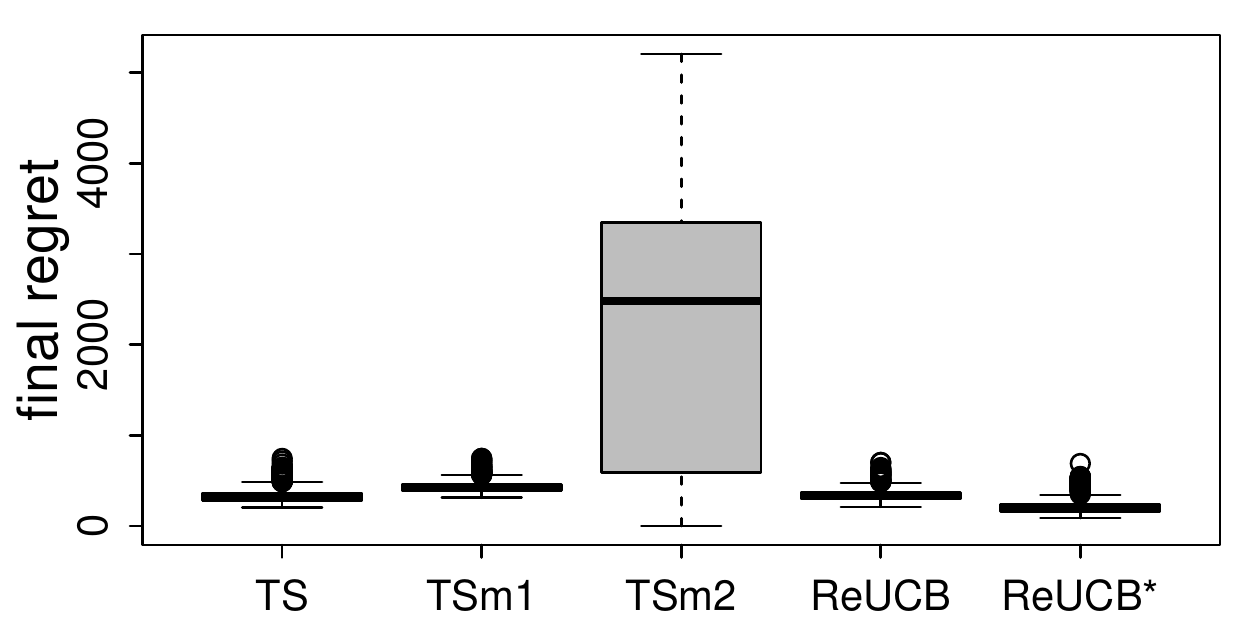}
\vspace{-0cm}
\subcaption{Gaussian $\mu_k\sim \mathcal{N}(1, 0.04)$.}
\end{subfigure}
\begin{subfigure}[b]{0.67\columnwidth}
\includegraphics[keepaspectratio,width=1\linewidth]{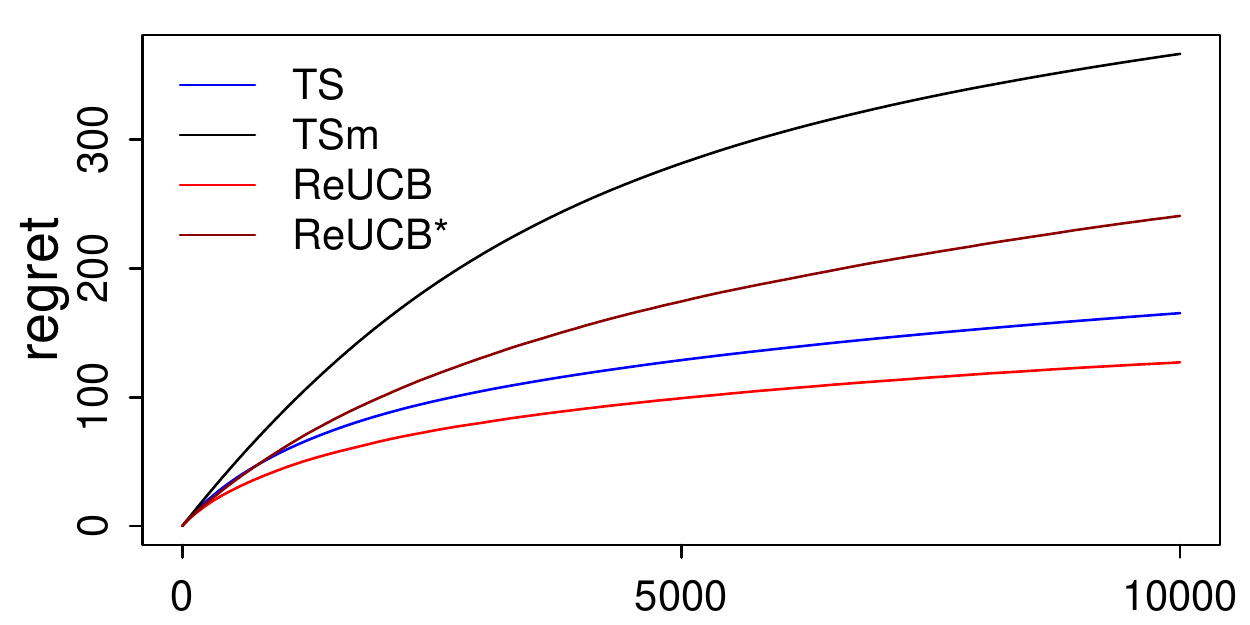}
\includegraphics[keepaspectratio,width=1\linewidth]{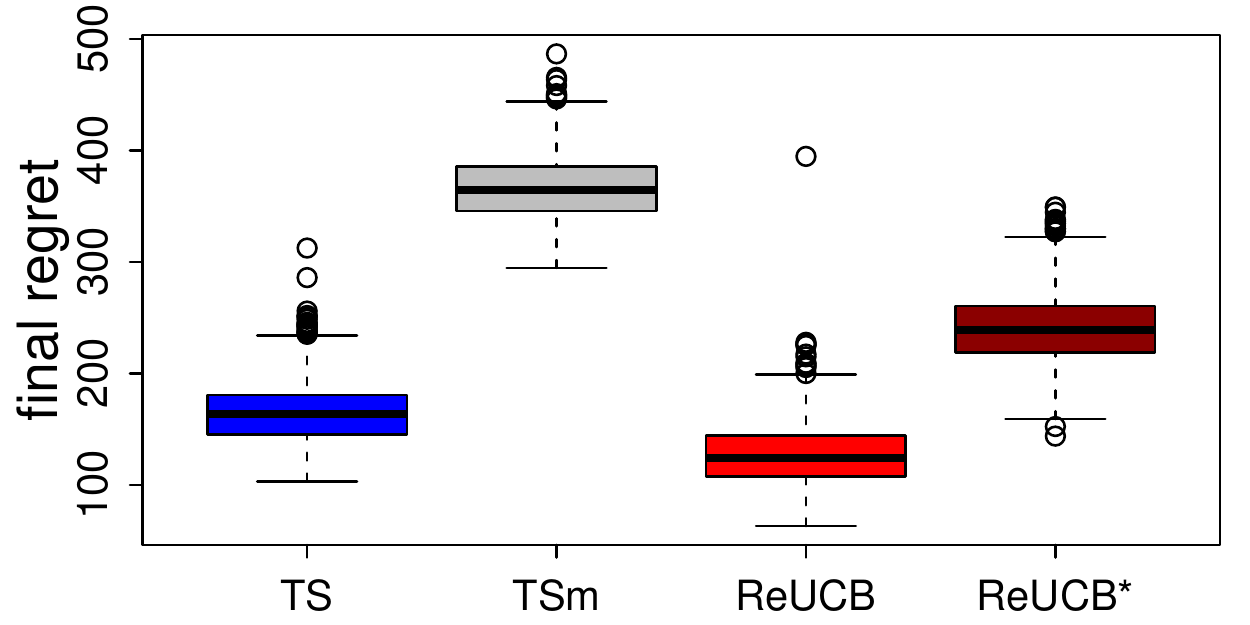}
\vspace{-0cm}
\subcaption{Bernoulli $\mu_k\sim \text{Beta}(1, 9)$.}
\end{subfigure}
\begin{subfigure}[b]{0.67\columnwidth}
\includegraphics[keepaspectratio,width=1\linewidth]{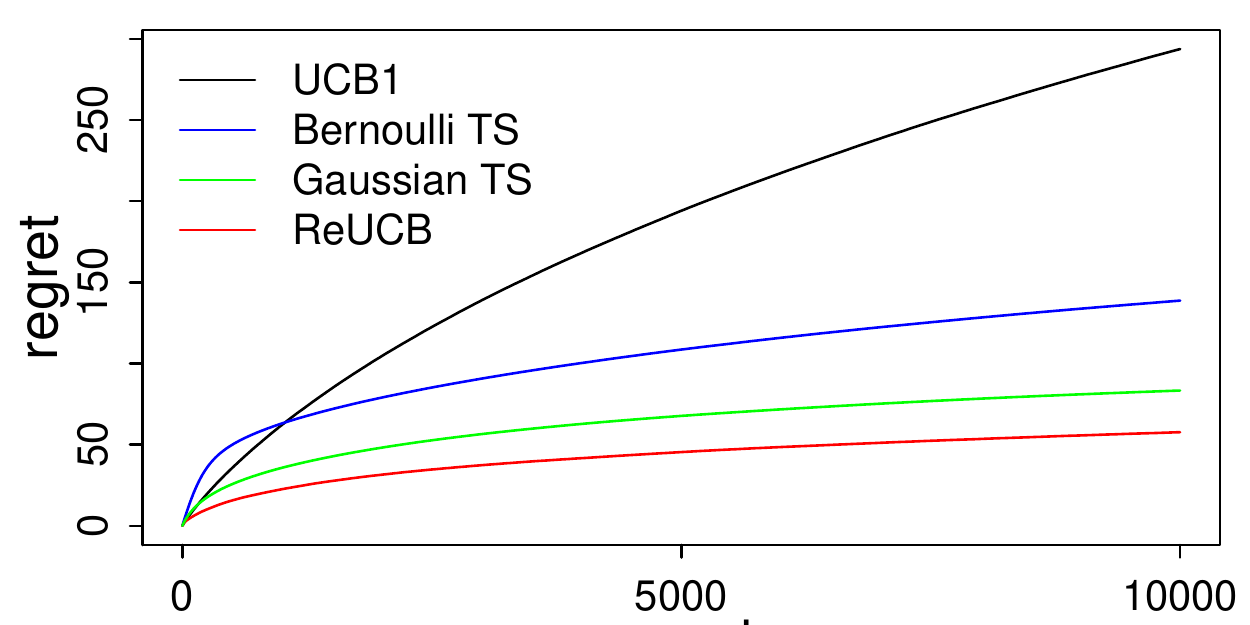}
\includegraphics[keepaspectratio,width=1\linewidth]{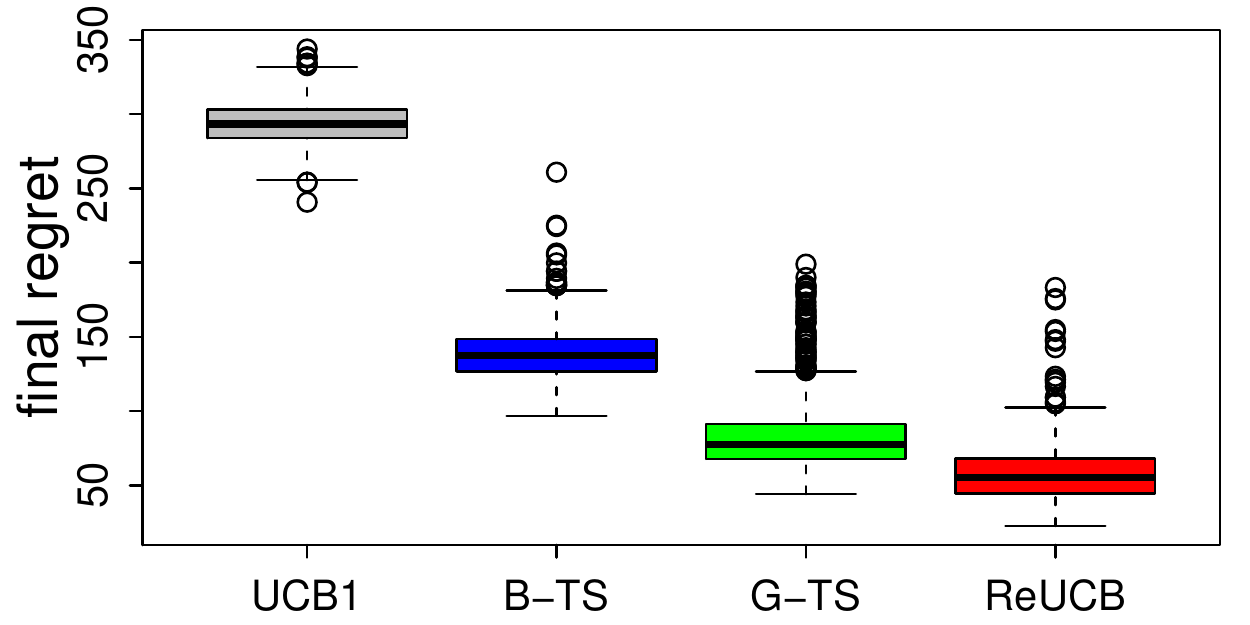}
\vspace{-0cm}
\subcaption{Truncated Gaussian}
\end{subfigure}
\vspace{-0cm}
\caption{Model misspecification experiments. Upper row: Regret as a function of round $n$. Lower row: Regret at the final round.}
\vspace{-0in}
\label{fig:Mis-prior}
\end{figure*}

Now we study what happens when \ts and \reucbs are applied to misspecified models. Note that \reucb is also  misspecified in Section \ref{sec:bernoulli bandits}, where the reward noise in Bernoulli bandits depends on the mean of the arm, meaning that it is not identical across the arms.

In the first experiment, we have a $50$-armed Gaussian bandit with $\mu_k \sim \mathcal{N}(1, 0.04)$, as in \cref{fig:Gaussian-armed}(a). We implement two variants of Thompson sampling with misspecified priors: \tsma with prior $\mathcal{N}(1, 1)$ (misspecified $\sigma_0^2$) and \tsmb with prior $\mathcal{N}(0, 0.04)$ (misspecified $\mu_0$). In \reucbs, $\sigma_0^2 = 1$ and thus is also misspecified. Our results are reported in \cref{fig:Mis-prior}(a), where \tsmb fails and has linear regret. The reason is that the misspecified prior has low variance and is downwards biased. Therefore, any initially pulled suboptimal arm is likely to be pulled again. \tsma also performs much worse than \ts with the correct prior. In contrast, \reucb estimates the unknown mean $\mu_0$ and outperforms \ts that knows $\mu_0$. Even \reucbs with misspecified $\sigma_0^2$ is comparable to \ts with the correct prior.

In the second experiment, we have a Bernoulli bandit with $K = 20$ arms and $\mu_k \sim \text{Beta}(1, 9)$. We study two variants of Thompson sampling: \ts with a correct prior and \tsm with misspecified prior $\text{Beta}(9, 1)$. In \reucbs, we set $\sigma_0^2 = 0.00818$ to match the variance of $\text{Beta}(9, 1)$ and $\sigma^2 = 0.25$ because this is the maximum reward variance. This makes \reucbs misspecified. Our results are reported in \cref{fig:Mis-prior}(b) and are similar to \cref{fig:Mis-prior}(a). We observe that \tsm performs worse than \reucbs, and that \ts has much higher regret than \reucb.

In the last experiment, we study reward-model misspecification. We have a $20$-armed Gaussian bandit where the rewards are truncated to $[0, 1]$ as $r_{k, j} = \min \{\max \{r'_{k, j}, 0\}, 1\}$ for $r'_{k, j} \sim \mathcal{N}(\mu_k, 0.04)$. The mean arm rewards are generated as $\mu_k \sim \mathcal{N}(0.3, 0.01)$. We implement Bernoulli \ts with prior $\text{Beta}(6, 14)$ to match the moments of $\mathcal{N}(0.3, 0.01)$ and Gaussian \ts with the correct prior $\mathcal{N}(0.3, 0.01)$. Our results are reported in \cref{fig:Mis-prior}(c). \reucb performs robustly and outperforms Thompson sampling.

\section{EXPERIMENTS ON REAL DATA}
\label{sec:real}

\begin{figure*}[!ht]
\begin{minipage}[c]{0.775\textwidth}
\begin{subfigure}[b]{0.495\columnwidth}
    \includegraphics[width=1\textwidth]{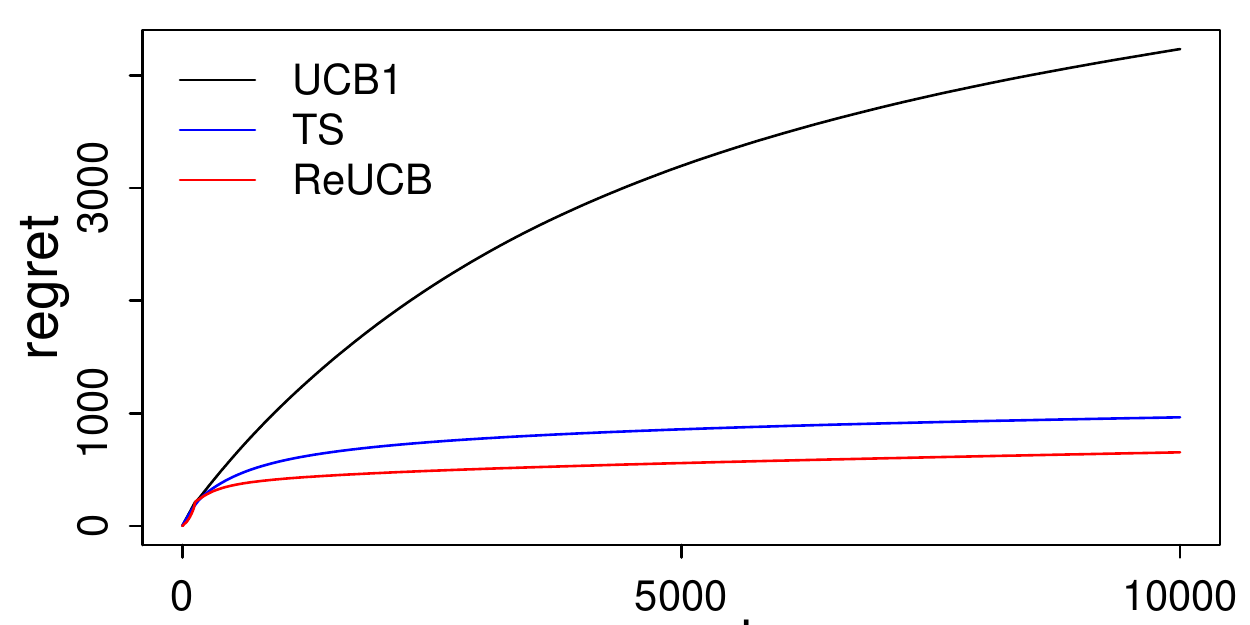}
\end{subfigure}
\begin{subfigure}[b]{0.495\columnwidth}
    \includegraphics[width=1\textwidth]{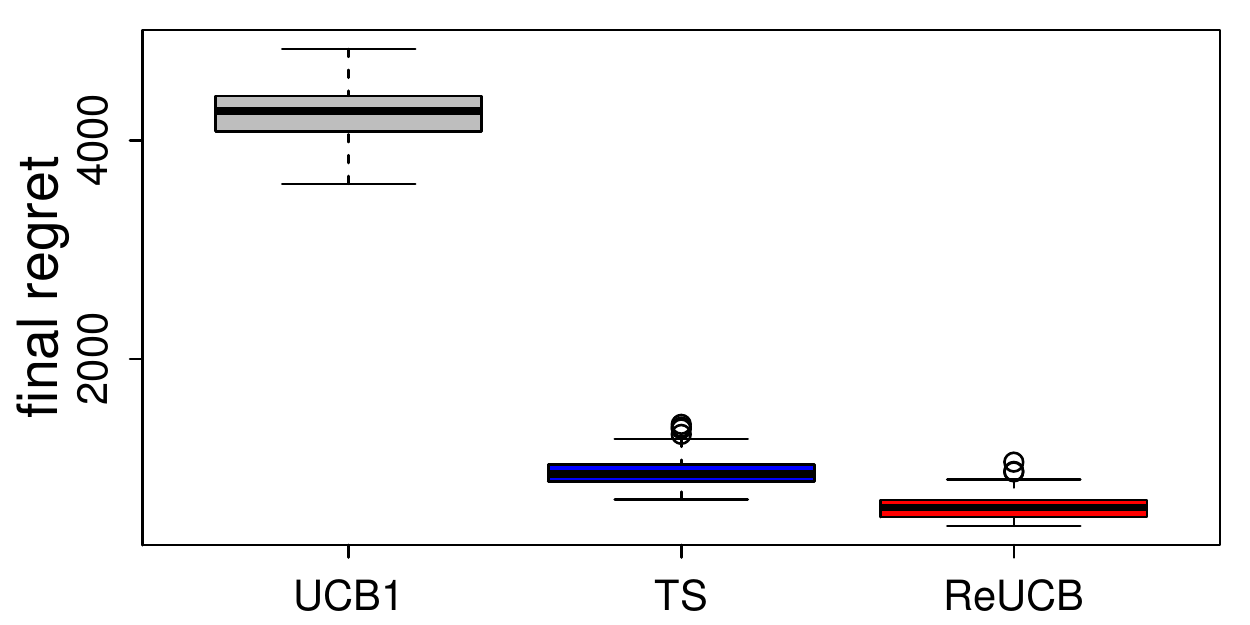}
\end{subfigure}
\end{minipage}\hfill
\begin{minipage}[c]{0.22\textwidth}
\caption{MovieLens experiment. Left: Regret as a function of round $n$. Right: Distribution of the regret at the final round.}
\label{fig:movie}
\end{minipage}
\vspace{-0cm}
\end{figure*}

In the last experiment, we evaluate \reucb on a recommendation problem. The goal is to identify the movie that has the highest expected rating. We experiment with the MovieLens dataset \citep{movie:16}, where we used a subset of $K = 128$ user groups and $L = 128$ movies randomly chosen from the full dataset, as described in \citet{K:17}. For each user group and movie, we average the ratings of all users in the group that rated the movie, and obtain the expected rating matrix $\mathbf{M}$ of rank $5$, which is learned by a low-rank approximation on the underlying rating matrix of the user groups and movies. See details of the pre-processing in \citet{K:17}. 

Our results are averaged over $200$ runs. In each run, user $j$ is chosen uniformly at random from $[128]$ and it represents a bandit instance in that run. The goal is to learn the most rewarding movie for user $j$. We treat this problem as a random-effect bandit with $K = 128$ arms, one per movie, where the mean reward of movie $k$ by user $j$ is $M_{j, k}$. The rewards are generated from $\mathcal{N}(M_{j, k}, 0.796^2)$, where the variance $0.796^2$ is estimated from data.

Our approach is compared to \ucb and \ts. We implement Gaussian \ts with a prior $\mathcal{N}(\mu_0, \sigma_0^2)$ that is estimated from the empirical mean rewards of all $128$ arms. That is, for each user $j$, $\mu_0$ and $\sigma_0^2$ are the empirical mean and variance of $M_{j,1}, \dots, M_{j,128}$. We implement \ucb by taking the upper confidence bound with $\sigma = 0.796$. Our results are reported in \cref{fig:movie}. We observe that \reucb has a much lower regret than both \ucb and \ts. This indicates that \reucb can learn the biases of different bandit instances, which represent individual users.

%% file: Conclusions.tex
\section{CONCLUSIONS}
\label{sec:conclusions}

We propose a random-effect bandit, a novel setting where the arm means are sampled i.i.d.\ from an unknown distribution. Using this model, we obtain an improved estimator of arm means and design an efficient UCB-like algorithm \reucb. \reucb is prior-free and we show empirically that it can outperform Thompson sampling. We analyze \reucb and prove a Bayes regret bound on its $n$-round regret, which improves over not using the random-effect structure.

Our initial results with random-effect models are encouraging. One limitation of our current approach is that \reucb is not contextual. In the future work, we plan to propose random-effect contextual bandits and provide an algorithm for them. Another limitation is that our regret analysis is under the assumption that \reucb knows $\sigma_0^2$ and $\sigma^2$. While this seems limiting, it is a weaker assumption than knowing the common mean $\mu_0$, which would be a standard assumption in the analysis of TS and \bayesucb.

%% file: Appendix.tex



\section{Derivation of \eqref{r0-estimator}}
\label{sec:r0-estimator}

Let $\mathbf{r}_k=(r_{k,1}, r_{k,2}, \dots, r_{k,n_k})^{\top}$ be a column vector of rewards obtained by pulling arm $k$. From modeling assumptions \eqref{payoff-M1} and \eqref{payoff-R}, we get
$$\mathbf{r}_k=\mu_k\mathbf{1}_k+\mathbf{e}_k\,,$$ 
where $\mathbf{e}_k=(e_{k,1}, e_{k,2}, \dots, e_{k,n_k})^{\top}$ and $e_{k,j}\sim P_r(0, \sigma^2)$. The covariance matrix for vector $\mathbf{r}_k$ is $\mathbf{V}_k=\sigma_0^2\mathbf{1}_k\mathbf{1}_k^{\top}+\sigma^2\mathbf{I}_k$, where $\mathbf{1}_k$ is an all-ones vector of length $n_k$ and $\mathbf{I}_k$ is a $n_k \times n_k$ identity matrix. The generalized least squares estimator of $\mu_0$ minimizes the following loss
\begin{align*}
  L(\mu_0)
  = \sum_{k = 1}^K (\mathbf{r}_k - \mu_0 \mathbf{1}_k)^\top
  \mathbf{V}_k^{-1} (\mathbf{r}_k - \mu_0 \mathbf{1}_k)
\end{align*}
with respect to $\mu_0$. Using the Sherman-Morrison formula,
\begin{align*}
\mathbf{V}_k^{-1}& =
\sigma^{-2}\mathbf{I}_k-\sigma^{-4}(\sigma_0^{-2}+n_k\sigma^{-2})^{-1}\mathbf{1}_k\mathbf{1}_k^{\top} 
= \sigma^{-2}\mathbf{I}_k-\sigma^{-2}\sigma_0^2(\sigma^{2}+n_k\sigma_0^{2})^{-1}\mathbf{1}_k\mathbf{1}_k^{\top}\notag\\
& =\sigma^{-2}\mathbf{I}_k-n_k^{-1}\sigma^{-2}w_{k}\mathbf{1}_k\mathbf{1}_k^{\top}\,.
\end{align*}
Inserting the above formula into $L(\mu_0)$ yields
\begin{align*}
L(\mu_0)& =\sigma^{-2}\sum\limits_{k=1}^K
\left[\|\mathbf{r}_k - \mu_0 \mathbf{1}_k\|^2-n_kw_k(\bar{r}_k-\mu_0)^2\right]\,.
\end{align*}
The first-order derivative of $L(\mu_0)$ with respect to $\mu_0$ is 
$$\frac{\partial L(\mu_0)}{\partial \mu_0}
=2\sigma^{-2}\sum\limits_{k=1}^K(1-w_{k})n_k(\bar{r}_k-\mu_0)\,.
$$
Thus we get that 
$\mu_0$ is estimated by 
\begin{equation}\label{r0-estimator-appendix}
\bar{r}_{0}
= \frac{\sum\limits_{k=1}^K(1-w_{k})n_{k}\bar{r}_{k}}{\sum\limits_{k=1}^K(1-w_{k})n_{k}}\,.
\end{equation}

\section{Derivation of \eqref{MSE}}
\label{sec:derivation-mse}

Now we derive the variance of $\hat{\mu}_{k}$. 
We have 
\begin{align}\label{MSE0}
\text{E}[(\hat{\mu}_{k}-\mu_k)^2] = \text{E}[(\tilde{\mu}_{k}-\mu_k+\hat{\mu}_{k}-\tilde{\mu}_{k})^2] = \text{E}[(\tilde{\mu}_{k}-\mu_k)^2]+\text{E}[(\hat{\mu}_{k}-\tilde{\mu}_{k})^2]\,,
\end{align}
where, we recall, $\tilde{\mu}_{k}$ in   \eqref{mu-estimator} and $\hat{\mu}_k$ in   \eqref{blup} differ only in that $\mu_0$ is estimated by $\bar{r}_0$, and  
the first step follows from $\text{E}[(\tilde{\mu}_{k}-\mu_k)(\hat{\mu}_{k}-\tilde{\mu}_{k})]=0$ \citep{KH:84}. 
Now we derive the two terms of the right-hand of  \eqref{MSE0}. 
Note that the first term is shown in  \eqref{MSE-term1}.
For the other term,
\begin{align}\label{MSE-term2-v0}
\text{E}[(\hat{\mu}_{k}-\tilde{\mu}_{k})^2]
&=(1-w_{k})^2\text{E}[(\bar{r}_{0}-\mu_{0})^2]\,.
\end{align}
From $\bar{r}_0$ in \eqref{r0-estimator}, 
\begin{align*}
\text{Var}(\bar{r}_{0})& =\left[\sum\limits_{k=1}^K(1-w_{k})n_{k}\right]^{-2}\sum\limits_{k=1}^K(1-w_{k})^2n_{k}^2\text{Var}(\bar{r}_{k})\notag\\
& =\left[\sum\limits_{k=1}^K(1-w_{k})n_{k}\right]^{-2}\sum\limits_{k=1}^K(1-w_{k})^2n_{k}^2(\sigma_0^2+\sigma^2/n_k)\notag\\
& =\sigma^2\left[\sum\limits_{k=1}^K(1-w_{k})n_{k}\right]^{-1}\,,
\end{align*}
where the last step is from $\sigma^2+n_k\sigma_0^2=(1-w_k)^{-1}\sigma^2$. 
Inserting the result above into \eqref{MSE-term2-v0}, we have 
\begin{align}\label{MSE-term2}
\text{E}[(\hat{\mu}_{k}-\tilde{\mu}_{k})^2]
&=\frac{(1-w_{k})^2\sigma^2}{\sum\limits_{k=1}^Kn_{k}(1-w_{k})}.
\end{align}
Therefore, inserting \eqref{MSE-term1} \& \eqref{MSE-term2} into \eqref{MSE0}, 
\begin{align}\label{MSE-appendix}
\text{E}[(\hat{\mu}_{k}-\mu_k)^2]
&=w_{k}n_k^{-1}\sigma^2+\frac{(1-w_{k})^2\sigma^2}{\sum\limits_{k=1}^Kn_{k}(1-w_{k})}
=:\tau_{k}^2,
\end{align}
where the reason that we use the squares notation in $\tau_k^2$ is because $\tau_k^2$ is a mean squared error not smaller than 0.

\section{Proofs of Proposition \ref{compare-tauANDmab}}
\label{sec:proof-prop}

Note that 
\begin{align*}
\sigma^2/n_k-\tau_k^2
= (1-w_{k})n_k^{-1}\sigma^2-\frac{(1-w_{k})^2\sigma^2}{\sum\limits_{i=1}^Kn_{i}(1-w_{i})}
= (1-w_{k})n_k^{-1}\sigma^2\left[1-\frac{n_k(1-w_{k})}{\sum\limits_{i=1}^Kn_{i}(1-w_{i})}\right]\,.
\end{align*}
Obviously, $n_k(1-w_{k})<\sum\limits_{i=1}^Kn_{i}(1-w_{i})$ as long as $n_i\geq 1$ for all $i \in [K]$. Thus, when $\sigma^2>0$, we get $\tau_k^2< \sigma^2/n_k$.

\section{Estimation of $\sigma_0^2$ and $\sigma^2$}
\label{sec:estimation-3}


Our BLUP estimators depend on $\sigma_0^2$ and $\sigma^2$, which may be unknown. We can estimate these quantities, and replace $\sigma_0^2$ and $\sigma^2$ in $w_{k}$ with these estimates. Various methods for obtaining consistent estimators of $\hat{\sigma}_0^2$ and $\hat{\sigma}^2$ are available, including the method of moments, maximum likelihood, and restricted maximum likelihood. See \citet{Robinson:91} for details. 

We use the method of moments, which does not rely on the assumption of distributions. 
Unbiased quadratic estimates of $\sigma^2$ and $\sigma_0^2$ are given by 
\begin{align}
\hat{\sigma}^2& =\left[\sum_{k=1}^K(n_{k}-1)\right]^{-1}\sum_{k=1}^K\sum\limits_{j=1}^{n_{k}}(r_{k,j}-\bar{r}_{k})^2\label{variance1}
\end{align}
and
\begin{align}
 \hat{\sigma}_0^2 &
 =n_*^{-1}\sum_{k=1}^K n_{k}u_k^2\,,
 \label{variance2}
\end{align}
where 
$n_*=\sum\limits_{k=1}^Kn_{k}-\left(\sum\limits_{k=1}^Kn_{k}\right)^{-1}\sum\limits_{k=1}^Kn_{k}^2$ 
and $u_k=\bar{r}_{k}-\left(\sum\limits_{k=1}^Kn_{k}\right)^{-1}\sum\limits_{k=1}^K\sum\limits_{j=1}^{n_{k}}r_{k,j}$\,.

\section{Varying Reward Noise}
\label{sec:varying}
We can generalize the standard random effect model in \eqref{payoff-R} by eliminating the assumption of identical observation noise across all arms. Instead, we allow the noise vary across arms. Specifically, the reward of arm $k$ after the $j$-th pull is assumed to be generated i.i.d.\ as
$$r_{k,j}\sim \mathcal{N}(\mu_k, \sigma_k^2)\,,$$ 
where, compared to model \eqref{payoff-R}, variance $\sigma_k^2$ is allowed to depend on $k$. 
Accordingly, we have the estimate $\hat{\mu}_{k}^h=(1-w_{k}^h)\bar{r}_0+w_{k}^h\bar{r}_{k}$ and $\hat{\mu}_{k}^h-\mu_{k}\mid H_t \sim \mathcal{N}(0,\tau_{h;k,t}^2)$, 
where the superscript ``h" means the heteroscedasticity of reward noise among arms,
$$w_k^h=\sigma_0^2/(\sigma_0^2+\sigma_k^2/n_k) \text{ and } \tau_{h;k,t}^2=w_{k}^hn_k^{-1}\sigma_k^2+(1-w_{k}^h)^2\sigma_k^2/\sum_{k=1}^Kn_{k}(1-w_{k}^h).$$  
We consider an upper bound of these $\sigma_k^2$, e.g., $\max_k\sigma_k^2$, denoted by $\sigma^2$.
Notice 
$$\tau_{h;k,t}^2=n_k^{-1}\sigma_k^2\sigma_0^2/(\sigma_0^2+n_k^{-1}\sigma_k^2)+(1-w_{k}^h)^2\sigma_0^2/(\sum_{k=1}^Kw_{k}^h).$$ 
Because $n_k^{-1}\sigma_k^2\sigma_0^2/(\sigma_0^2+n_k^{-1}\sigma_k^2)$ is increasing of $\sigma_k^2$ and $w_k^h$ is decreasing of $\sigma_k^2$, 
we have that 
$$\tau_{h;k,t}^2\leq \tau_{k,t}^2.$$  
Therefore, we uniformly use the upper variance $\sigma^2$ across arms replacing of $\sigma_k^2$. By this way, the Bayes regret bound in Theorem \ref{them-R-Bayes} still holds.

\section{Lemmas}
\label{sec:lemma}

\begin{lemma}
\label{sec:tau} 
We have that 
\begin{align*}
\tau_{k}^2
\leq \frac{\sigma_0^2\sigma^2}{n_k\sigma_0^2+\sigma^2}(1+K^{-1}\sigma^2\sigma_0^{-2})\,.
\end{align*} 
\end{lemma}

\begin{proof}
Now we provide an upper bound on $\tau_k^2$. 
By using $n_k(1-w_k)\sigma_0^2=w_k\sigma^2$, we have
\begin{align*}
\tau_{k}^2&=w_{k}n_k^{-1}\sigma^2+\frac{(1-w_{k})^2\sigma^2}{\sum\nolimits_{k=1}^Kn_{k}(1-w_{k})}\notag\\
&=w_{k}n_k^{-1}\sigma^2+\frac{(1-w_{k})^2\sigma_0^2}{\sum\nolimits_{k=1}^Kw_{k}}\notag\\
&\leq w_{k}n_k^{-1}\sigma^2+K^{-1}(1-w_{k})^2(\sigma_0^2+\sigma^2)\notag\\
&= w_{k}n_k^{-1}\sigma^2+K^{-1}n_k^{-1}(1-w_{k})w_{k}\sigma^2\sigma_0^{-2}(\sigma_0^2+\sigma^2)\notag\\
&\leq w_{k}n_k^{-1}\sigma^2(1+K^{-1}\sigma^2\sigma_0^{-2})\notag\\
&= \frac{\sigma_0^2\sigma^2}{n_k\sigma_0^2+\sigma^2}(1+K^{-1}\sigma^2\sigma_0^{-2})\,,
\end{align*}
where the first inequality is from $w_k\geq \sigma_0^2/(\sigma_0^2+\sigma^2)$ and 
the last inequality is from $1-w_k\leq \sigma^2/(\sigma_0^2+\sigma^2)$ due to $n_k \geq 1$ for all $k\in[K]$. 
\end{proof}


\begin{lemma}\label{sec:mu0Dist}
Let
$r_{k, j}\sim \mathcal{N}(\mu_k, \sigma^2)$ and $\mu_k\sim\mathcal{N}(\mu_0, \sigma_0^2)$.  
Assuming that $\sigma^2$ and $\sigma_0^2$ are known, and $\mu_0$ is an improper flat prior, i.e., $p(\mu_0)\propto 1$, we have that $\mu_{k}\mid H_t  \sim \mathcal{N}(\hat{\mu}_{k,t}, \tau_{k,t}^{2})$. 
\end{lemma}


\begin{proof}

Recall the following well-known identity. Let $Y|X=x\sim \mathcal{N}(ax+b,\sigma^2)$ and $X\sim \mathcal{N}(\mu,\sigma_x^2)$. Then
$$
Y\sim \mathcal{N}(a\mu+b,a^2\sigma_x^2+\sigma^2).
$$
Obviously, if we set $X$ to $\mu_0\mid H_t$ and $Y$ to $\mu_k\mid H_t$, we can apply this result to obtain the distribution of $\mu_k\mid H_t$ from the distributions of $\mu_k\mid \mu_0, H_t$ and $\mu_0\mid H_t$. 
The distribution of $\mu_k\mid \mu_0, H_t$ is studied in \cref{sec:estimation-1}. Under the assumptions of Lemma \ref{sec:mu0Dist}, $\mu_k\mid \mu_0, H_t$ is a Gaussian with mean in \eqref{mu-estimator} and variance in \eqref{MSE-term1}. 

Now we derive the distribution of $\mu_0\mid H_t$. Note that  $p(\mu_0) \propto 1$ is the extreme case of $\mu_0 \sim \mathcal{N}(0, \lambda)$ as $\lambda \to \infty$. 
Assuming $\mu_0 \sim \mathcal{N}(0, \lambda)$, 
$\bar{r}_k$ can be considered to be generated from the following Bayesian model:
$$
\bar{r}_k\mid \mu_0 \sim \mathcal{N}(\mu_0, \sigma_0^2 + \sigma^2 / n_k), \quad
\mu_0 \sim \mathcal{N}(0, \lambda).
$$
Thus, the distribution of $\mu_0\mid H_t$ is easily obtained as 
$$
\mu_0\mid H_t \sim \mathcal{N}(\bar{r}_0, \sigma_0^2),
$$ 
where
$$
\bar{r}_0
= \sigma_0^2 [\sigma^{-2} \sum_{k = 1}^K (1 - w_k) n_k \bar{r}_k] \quad \text{and}\quad 
\sigma_0^2
= [\lambda^{-1} + \sigma^{-2} \sum_{k = 1}^K (1 - w_k) n_k]^{-1}.
$$
Taking $\lambda \to \infty$, we have that
$$
\mu_0\mid H_t
\sim \mathcal{N}(\bar{r}_0, \sigma^2 [\sum_{k = 1}^K (1 - w_k) n_k]^{-1}).
$$
Using the above results, we obtain the distribution of $\mu_k\mid H_t$ from the distributions of $\mu_k\mid \mu_0, H_t$ and $\mu_0\mid H_t$. That distribution is a Gaussian with mean in \eqref{blup} and variance in \eqref{MSE}. This completes the proof.
\end{proof}

\section{Extension to sub-Gaussian}
\label{sec:extensions}


Our analysis can be extended to bounded sub-Gaussian random variables. Without loss of generality, we consider the support of $[0, 1]$ below.

\begin{theorem}\label{them-R-Bayes-SubR}
Consider \reucb in a $K$-armed bandit with sub-Gaussian rewards $r_{k, j} - \mu_k \sim \text{subG}(\sigma^2)$ and  $\mu_k-\mu_0\sim\text{subG}(\sigma_0^2)$ with support in $[0,1]$. Let $\sigma_0^2$ and $\sigma^2$ be known and used by \reucb. Define $m = [1+K^{-1}\sigma_0^2/(\sigma_0^2+\sigma^2)]^{-1}[1+K^{-1/2}\sigma/\sigma_0]^2$. Then
(1) for any $ a\geq m$, the $n$-round Bayes regret of \reucb is
\begin{align*}
R_n&\leq 2\left(1 + \frac{\sigma^2}{K \sigma_0^2}\right)\sqrt{
  \frac{a\sigma_0^2\log(1 + \sigma^{-2} \sigma_0^2 n)}{\log(1 + \sigma^{-2} \sigma_0^2)}K n\log n}
  + \frac{K\sigma_0^2+\sigma^2}{\sigma_0^2}\sqrt{\frac{8n\sigma_0^{2}\sigma^2}{\pi(\sigma_0^2+\sigma^2)}}\,.
\end{align*}
(2) for any $ a \geq 2m$, the $n$-round Bayes regret of \reucb is obtained by replacing the last term above with $2K(1+\log n)$.
\end{theorem}

\begin{proof}
Under the sub-Gaussian assumptions that $\mu_k-\mu_0\sim \text{subG}(\sigma_0^2)$ and
$\bar{r}_{k,t}-\mu_k \sim \text{subG}(\sigma^2/n_k)$, Lemma \ref{sec:tau_star} shows that $\mu_{k}-\hat{\mu}_{k,t} \sim \text{subG}(\tau_{k,t}^{*2})$ for $k\in [K]$, 
where  
\begin{align}
\tau_{k,t}^{*2}&=\sigma^2\left[\sqrt{\frac{w_k}{n_k}}+\sqrt{\frac{(1-w_k)^2}{\sum\nolimits_{k=1}^K(1-w_k)n_k}}\right]^2.\label{tau_SG} 
\end{align}
Lemma \ref{sec:tau*} of the Appendix shows that 
\begin{align}
\tau_{k,t}^{*2}
&\leq \frac{\sigma_0^2\sigma^2}{n_k\sigma_0^2+\sigma^2}\left(1+\sqrt{K^{-1}\sigma^2\sigma_0^{-2}}\right)^2\label{tau_SG-ineq}. 
\end{align}

Note that $c_{k,t}=\sqrt{2\tau_{k,t}^{2}\log(1/\delta_t)}$, $E_{R;t}$, and $E_{L;t}$. 
We have that 
\begin{align}\label{decom1-result2}
\condE{(\hat{\mu}_{I_t,t}-\mu_{I_t})\mathds{1}\{\bar{E}_{R;t}\}}{H_t}
& \leq \condE{\mathds{1}\{\bar{E}_{R;t}\}}{H_t},
\end{align}
where the inequality is from 
$(\hat{\mu}_{I_t,t}-\mu_{I_t})\in[0,1]$ on $\bar{E}_{R;t}$ due to the support $[0,1]$. 
It follows that when $a\geq m$, 
\begin{align*}
\E{\sum\limits_{t=1}^n\condE{(\hat{\mu}_{I_t,t}-\mu_{I_t})\mathds{1}\{\bar{E}_{R;t}\}}{H_t}}
& \leq \E{\sum\limits_{t=1}^n\condE{\mathds{1}\{\bar{E}_{R;t}\}}{H_t}}
= \sum\limits_{t=1}^n\E{\mathds{1}\{\bar{E}_{R;t}\}}\notag\\
&\leq K\sum\limits_{t=1}^n\delta_t^{1/m}\leq K\sum\limits_{t=1}^nt^{-1/2} 
\leq 2K\sqrt{n}, 
\end{align*}
where the second inequality is from $\tau_{k,t}^{2}/\tau_{k,t}^{*2}\geq 1/m$ shown in Lemma \ref{lemma:tauTotau}, and the third inequality is from the $a\geq m$. 
Similarly we have 
$\E{\sum\limits_{t=1}^n\condE{(\mu_{I_*}-\hat{\mu}_{I_*,t})\mathds{1}\{\bar{E}_{L;t}\}}{H_t}}
 \leq 2K\sqrt{n}$.

Similarly to \eqref{regret-decom-upper1}, we have that 
$$\E{\sum\limits_{t=1}^n\sqrt{2\tau_{I_t,t}^{2}\log(1/\delta_t)}}\leq \left(1 + \frac{\sigma^2}{K \sigma_0^2}\right)\sqrt{
  \frac{a\sigma_0^2\log(1 + \sigma^{-2} \sigma_0^2 n)}{\log(1 + \sigma^{-2} \sigma_0^2)}
   K n\log n}.$$
Therefore, when $a\geq m$,  the regret is bounded as 
\begin{align*}
\E{R_n} 
&\leq 2\left(1 + \frac{\sigma^2}{K \sigma_0^2}\right)\sqrt{
  \frac{a\sigma_0^2\log(1 + \sigma^{-2} \sigma_0^2 n)}{\log(1 + \sigma^{-2} \sigma_0^2)} K n\log n}+4K\sqrt{n}.
\end{align*}
Similarly, when $a\geq 2m$,  the regret is bounded as 
\begin{align*}
\E{R_n} 
&\leq 2\left(1 + \frac{\sigma^2}{K \sigma_0^2}\right)\sqrt{
  \frac{a\sigma_0^2\log(1 + \sigma^{-2} \sigma_0^2 n)}{\log(1 + \sigma^{-2} \sigma_0^2)} K n\log n}+4K(1+\log n).
\end{align*}
\end{proof}

When comparing \cref{them-R-Bayes,them-R-Bayes-SubR}, the regret is of the same order. Since the assumption of $r_{k, j} - \mu_k \sim \text{subG}(\sigma^2)$ allows for modeling arm-dependent reward noise, such as $\sigma^2 = 1 / 4$ in Bernoulli bandits, \cref{them-R-Bayes-SubR} holds for Bernoulli bandits. In \cref{sec:experiments}, we experiment with Bernoulli bandits. Unlike \cref{them-R-Bayes}, \cref{them-R-Bayes-SubR} requires that $a\geq m$ or $a\geq 2m$. We note that $m$ in \cref{them-R-Bayes-SubR} is typically small, and approaches $1$ as $K$ and $\sigma_0^2/\sigma^2$ increase.

\begin{lemma}\label{sec:tau_star}
We have that for $k\in [K]$
$$\mu_{k}-\hat{\mu}_{k,t} \sim \text{subG}(\tau_{k,t}^{*2}),$$ where  
\begin{align*}
\tau_{k,t}^{*2}&=\sigma^2\left[\sqrt{\frac{w_k}{n_k}}+\sqrt{\frac{(1-w_k)^2}{\sum\nolimits_{k=1}^K(1-w_k)n_k}}\right]^2. 
\end{align*}
\end{lemma}

Notice 
\begin{align*}
\hat{\mu}_{k,t}-\mu_k & =\tilde{\mu}_{k,t}-\mu_k+\hat{\mu}_{k,t}-\tilde{\mu}_{k,t}, 
\end{align*}
where $\tilde{\mu}_{k,t}-\mu_k=w_k(\bar{r}_{k,t}-\mu_k)+(1-w_k)(\mu_k-\mu_{0})$ and $\hat{\mu}_{k,t}-\tilde{\mu}_{k,t}= (1-w_k)(\bar{r}_{0,t}-\mu_0)$. 
From the properties of sub-Gaussian (Fact 2) and the independence between $\mu_0-\mu_k$ and $\bar{r}_{k,t}-\mu_k$,  
we have 
\begin{align*}
\tilde{\mu}_{k,t}-\mu_k  
& \sim \text{subG}\left(w_k\sigma^2/n_k\right); \notag\\
\bar{r}_{0,t}-{\mu}_{0}  
& \sim \text{subG}\left(\sigma^2[\sum\limits_{k=1}^K(1-w_k)n_k]^{-1}\right).
\end{align*}
Thus, the properties of sub-Gaussian (Fact 2) tell us that 
$$\tau_{k,t}^{*2}=\sigma^2\left[\sqrt{\frac{w_k}{n_k}}+\frac{1-w_k}{\sqrt{\sum\limits_{k=1}^K(1-w_k)n_k}}\right]^2,$$
i.e., 
$\hat{\mu}_{k,t}-\mu_k$ is sub-Gaussian with variance proxy $\tau_{k,t}^{*2}$. 

\begin{lemma}\label{sec:tau*}
\begin{align*}
\tau_{k,t}^{*2}&\leq  \frac{\sigma_0^2\sigma^2}{n_k\sigma_0^2+\sigma^2}\left(1+\sqrt{K^{-1}\sigma^2\sigma_0^{-2}}\right)^2. 
\end{align*}
\end{lemma}

\begin{proof}
Now we provide an upper bound on $\tau_k^{*2}$. 
By making using of $n_k(1-w_k)\sigma_0^2=w_k\sigma^2$, we have that 
\begin{align*}
\tau_{k,t}^{*2}&=\sigma^2\left[\sqrt{w_{k}n_k^{-1}}+\sqrt{\frac{(1-w_{k})^2}{\sum\nolimits_{k=1}^Kn_{k}(1-w_{k})}}\right]^2\notag\\
&=\left[\sqrt{\sigma^2w_{k}n_k^{-1}}+\sqrt{\frac{(1-w_{k})^2\sigma_0^2}{\sum\nolimits_{k=1}^Kw_{k}}}\right]^2\notag\\
&\leq\left[\sqrt{\sigma^2w_{k}n_k^{-1}}+\sqrt{K^{-1}(1-w_{k})^2(\sigma_0^2+\sigma^2)}\right]^2\notag\\
&=\left[\sqrt{\sigma^2w_{k}n_k^{-1}}+\sqrt{K^{-1}n_k^{-1}(1-w_{k})w_{k}\sigma^2\sigma_0^{-2}(\sigma_0^2+\sigma^2)}\right]^2\notag\\
&\leq \left[\sqrt{\sigma^2w_{k}n_k^{-1}}(1+\sqrt{K^{-1}\sigma^2\sigma_0^{-2}})\right]^2\notag\\
&= \frac{\sigma_0^2\sigma^2}{n_k\sigma_0^2+\sigma^2}\left(1+\sqrt{K^{-1}\sigma^2\sigma_0^{-2}}\right)^2,
\end{align*}
where the first inequality is from $w_k\geq \sigma_0^2/(\sigma_0^2+\sigma^2)$, and 
the last inequality is from $1-w_k\leq \sigma^2/(\sigma_0^2+\sigma^2)$.

\end{proof}

\begin{lemma}\label{sec:tau2}
\begin{align*}
\tau_{k,t}^2&\geq \frac{\sigma_0^2\sigma^2}{n_k\sigma_0^2+\sigma^2}(1+K^{-1}\sigma_0^2/(\sigma_0^{2}+\sigma^2)).
\end{align*}
\end{lemma}

\begin{proof}
Now we provide a lower bound on $\tau_k^2$. 
Similarly, we have that 
\begin{align*}
\tau_{k,t}^2&=w_{k}n_k^{-1}\sigma^2+\frac{(1-w_{k})^2\sigma_0^2}{\sum\nolimits_{k=1}^Kw_{k}}\notag\\
&\geq w_{k}n_k^{-1}\sigma^2+K^{-1}(1-w_{k})^2\sigma_0^2\notag\\
&= w_{k}n_k^{-1}\sigma^2+K^{-1}n_k^{-1}(1-w_{k})w_{k}\sigma^2\notag\\
&\geq w_{k}n_k^{-1}\sigma^2(1+K^{-1}\sigma_0^2/(\sigma_0^2+\sigma^2))\notag\\
&= \frac{\sigma_0^2\sigma^2}{n_k\sigma_0^2+\sigma^2}(1+K^{-1}\sigma_0^2/(\sigma_0^2+\sigma^2)),
\end{align*}
where the first inequality is from $w_k\leq 1$, and 
the last inequality is from $1-w_k\geq \sigma_0^2/(\sigma_0^2+\sigma^2)$.
\end{proof}

\begin{lemma}\label{lemma:tauTotau}
\begin{align*}
\frac{\tau_{k,t}^2}{\tau_k^{*2}}&\geq  \frac{1+\sigma_0^2/(K(\sigma_0^2+\sigma^2))}{(1+K^{-1/2}\sigma/\sigma_0)^2}.
\end{align*}
\end{lemma}
\begin{proof}
From Lemmas \ref{sec:tau2} \& \ref{sec:tau*}, we use the upper bound of $\tau_{k,t}^{*2}$ and the lower bound of $\tau_{k,t}^{2}$. Then the result is proved. 
\end{proof}




\section{Some Facts}
\label{sec:fact}

Let $X$ and $Y$ be sub-Gaussian with variance proxies $\sigma^2$ and $\tau^2$, respectively. Then 
(1) $aX$ is sub-Gaussian with variance proxy $a^2\sigma^2$; 
(2) $X+Y$ is sub-Gaussian with variance proxy $(\sigma+\tau)^2$; 
and (3) if $X$ and $Y$ are independent, $X+Y$ is sub-Gaussian with variance proxy $\sigma^2+\tau^2$. 

\section{Additional Experiments}
\label{sec:add-exp}

\begin{figure}[htbp]
\centering
\includegraphics[keepaspectratio,width=0.495\linewidth]{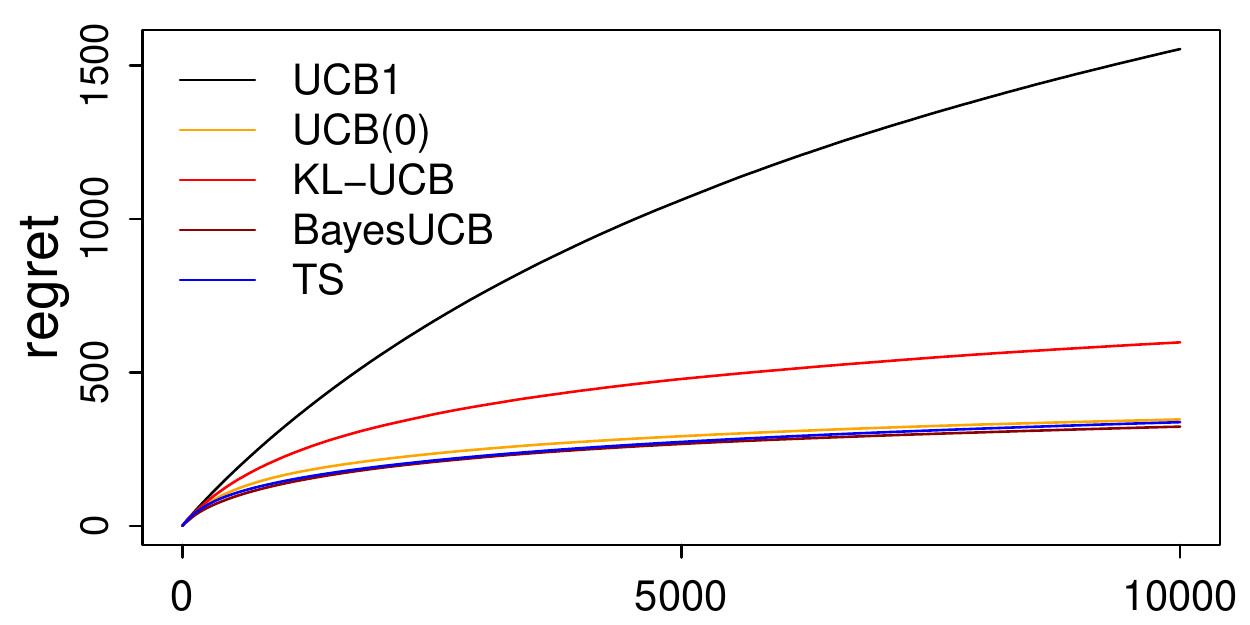}
\includegraphics[keepaspectratio,width=0.495\linewidth]{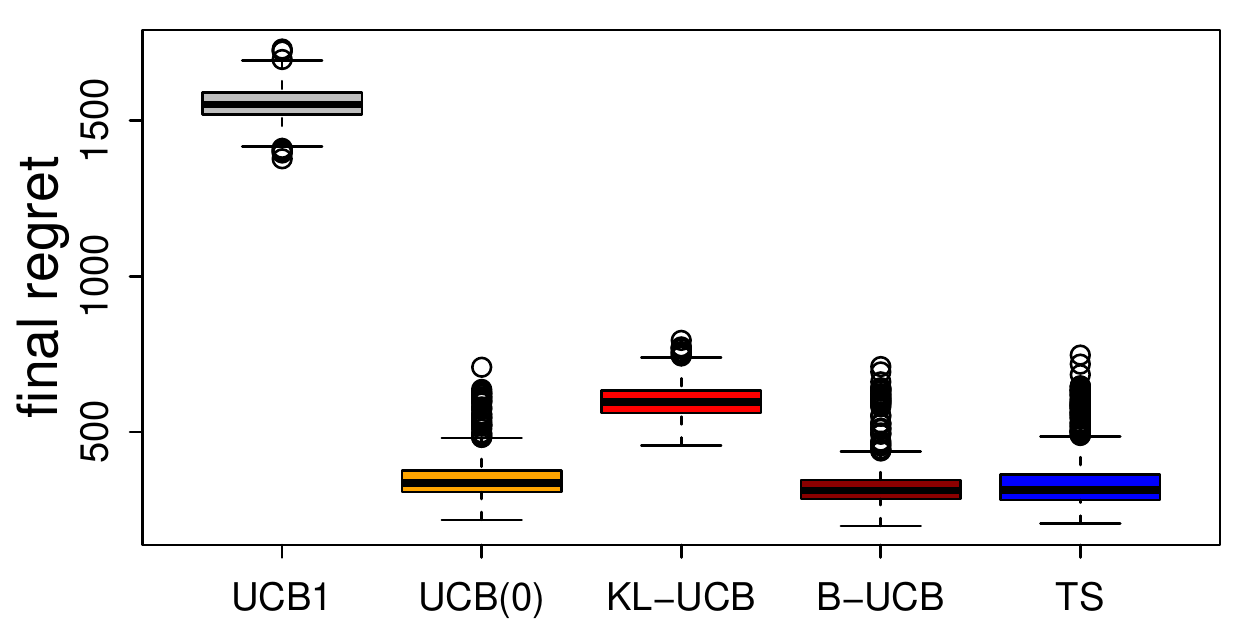}
\caption{Performance of UCB algorithms on the $50$-armed Gaussian bandit with $\mu_k \sim \mathcal{N}(1, 0.04)$, as in \cref{fig:Gaussian-armed}(a). Left column: 
Regret performance as a function of round $n$.
Right column: Distribution of the regret at the final round. 
``UCB(0)" denotes the extreme case of \reucb, i.e., \reucbsinf by taking $w_k=1$ that behaves as \ucb with $a=1$. ``B-UCB" in the right figure denotes BayesUCB. 
The results are summarized over 1000 runs. 
}
\label{fig:Gaussian-Others}
\vspace{-0.3cm}
\end{figure}




